\documentclass{article}

\usepackage{microtype}
\usepackage{graphicx}
\usepackage{subcaption}
\usepackage{booktabs} % for professional tables

\usepackage{hyperref}

\usepackage[preprint]{icml2026}

\usepackage{amsmath}
\usepackage{amssymb}
\usepackage{mathtools}
\usepackage{amsthm}

\usepackage[capitalize,noabbrev]{cleveref}

\usepackage{fontawesome} 
\usepackage{multirow}
\theoremstyle{plain}
\newtheorem{theorem}{Theorem}[section]
\newtheorem{proposition}[theorem]{Proposition}
\newtheorem{lemma}[theorem]{Lemma}

\theoremstyle{definition}

\theoremstyle{remark}
\newtheorem{remark}[theorem]{Remark}
\usepackage{enumerate}

\usepackage{amsmath,amssymb}
\usepackage{xcolor}
\usepackage{tikz}
\usetikzlibrary{arrows.meta,positioning,calc}

\definecolor{taskC}{RGB}{255,143,143}
\definecolor{entC}{RGB}{170,190,255}
\definecolor{balC}{RGB}{255,210,150}
\definecolor{taskD}{RGB}{200, 40, 40}
\definecolor{entD}{RGB}{40, 80,190}
\definecolor{balD}{RGB}{180,100, 20}
\usepackage[table]{xcolor}

\usepackage{pgfplots}
\pgfplotsset{compat=newest}
\definecolor{color1}{HTML}{84B818}
\definecolor{color2}{HTML}{D18B12}
\definecolor{color3}{HTML}{7F7F7F}
\definecolor{color4}{HTML}{F85A3E}
\definecolor{color5}{HTML}{4B6CFC}
\definecolor{color6}{HTML}{FF7F0E}
\definecolor{color7}{HTML}{2CA02C}
\definecolor{color8}{HTML}{D62728}
\definecolor{color9}{HTML}{9467BD}
\definecolor{color10}{HTML}{8C564B}
\definecolor{color11}{HTML}{E377C2}
\definecolor{color12}{HTML}{17BECF}
\newcommand\our{DSRD}
\newcommand\ours{Dual-Scale Retentive Dynamics}
    
\usepackage[textsize=tiny]{todonotes}

\icmltitlerunning{Forget Less, Generalize More: Unifying Temporal and Structural Adaptation for Dynamic Graphs}

\begin{document}

\twocolumn[
\icmltitle{Forget Less, Generalize More: Unifying Temporal and Structural Adaptation for Dynamic Graphs}

  % It is OKAY to include author information, even for blind submissions: the
  % style file will automatically remove it for you unless you've provided
  % the [accepted] option to the icml2026 package.

  % List of affiliations: The first argument should be a (short) identifier you
  % will use later to specify author affiliations Academic affiliations
  % should list Department, University, City, Region, Country Industry
  % affiliations should list Company, City, Region, Country

  % You can specify symbols, otherwise they are numbered in order. Ideally, you
  % should not use this facility. Affiliations will be numbered in order of
  % appearance and this is the preferred way.
  
% \icmlsetsymbol{equal}{*}

\begin{icmlauthorlist}
\icmlauthor{Qian Chang}{UoA}
\icmlauthor{Ciprian Doru Giurcaneanu}{UoA}
\icmlauthor{Runsong Jia}{UTS}
\icmlauthor{Xia Li}{CCNU}
\icmlauthor{Guoping Hu}{UoA}
\icmlauthor{Xiufeng Cheng}{CCNU}
\icmlauthor{Jinqing Yang}{CCNU}
\icmlauthor{Mengjia Wu}{UTS}
\icmlauthor{Yi Zhang}{UTS}

\end{icmlauthorlist}

\icmlaffiliation{UoA}{University of Auckland, Auckland, New Zealand}
\icmlaffiliation{UTS}{University of Technology Sydney, Sydney, Australia }
\icmlaffiliation{CCNU}{Central China Normal University, Wuhan, China}

\icmlcorrespondingauthor{Ciprian Doru Giurcaneanu}{c.giurcaneanu@auckland.ac.nz}
% \icmlcorrespondingauthor{Firstname2 Lastname2}{first2.last2@www.uk}

% You may provide any keywords that you
% find helpful for describing your paper; these are used to populate
% the "keywords" metadata in the PDF but will not be shown in the document
\icmlkeywords{Machine Learning, ICML}

\vskip 0.3in

]

% this must go after the closing bracket ] following \twocolumn[ ...

% This command actually creates the footnote in the first column
% listing the affiliations and the copyright notice.
% The command takes one argument, which is text to display at the start of the footnote.
% The \icmlEqualContribution command is standard text for equal contribution.
% Remove it (just {}) if you do not need this facility.

\printAffiliationsAndNotice{}  % leave blank if no need to mention equal contribution
%\printAffiliationsAndNotice{\icmlEqualContribution} % otherwise use the standard text.

\begin{abstract}
Representation learning on dynamic graphs requires capturing complex dependencies that evolve across both time and structure. Existing approaches typically adopt fixed temporal decay schemes or predetermined structural propagation depths, limiting their ability to generalize across graphs with diverse interaction frequencies and topological characteristics. We propose \ours~(\our), a unified framework that maintains a retentive representation state encoding both temporal memory and structural context. \our~introduces two key components: (i) a retentive state with dual-scale adaptation that jointly models temporal dynamics and structural propagation within a single recurrent formulation, and (ii) adaptive decay kernels with learnable time-sensitivity parameters that automatically balance short-term responsiveness and long-term retention based on the underlying interaction patterns. We provide theoretical analysis establishing the equivalence between event-wise parallel aggregation and efficient recurrent state updates, as well as stability and boundedness guarantees for the learned dynamics. Extensive experiments on 14 real-world benchmarks demonstrate that \our~consistently achieves state-of-the-art performance on both link prediction and node classification tasks, with strong generalization across transductive and inductive settings. The implementation code is available at \href{https://anonymous.4open.science/r/DSRD}{[\faGithub Code]}.
\end{abstract}

\section{Introduction}
\label{section-1}
Dynamic graphs (also referred to as temporal graphs) offer a principled abstraction for modeling real-world systems whose entities and relations evolve over time~\cite{holme2012temporal,feng2025comprehensive,martinez2016survey}, enabling structured reasoning over temporal interactions in domains such as social networks~\cite{deng2019learning}, traffic systems~\cite{li2022dmgan}, and knowledge graphs~\cite{wang2024large}. However, learning effective representations from such data remains challenging due to the tight entanglement between temporal dynamics and evolving topological structure, which constrains both the expressiveness and robustness of existing dynamic graph learning models~\cite{feng2025comprehensive, gravina2024deep}.
\begin{figure}[th]
    \centering
    \includegraphics[width=\linewidth]{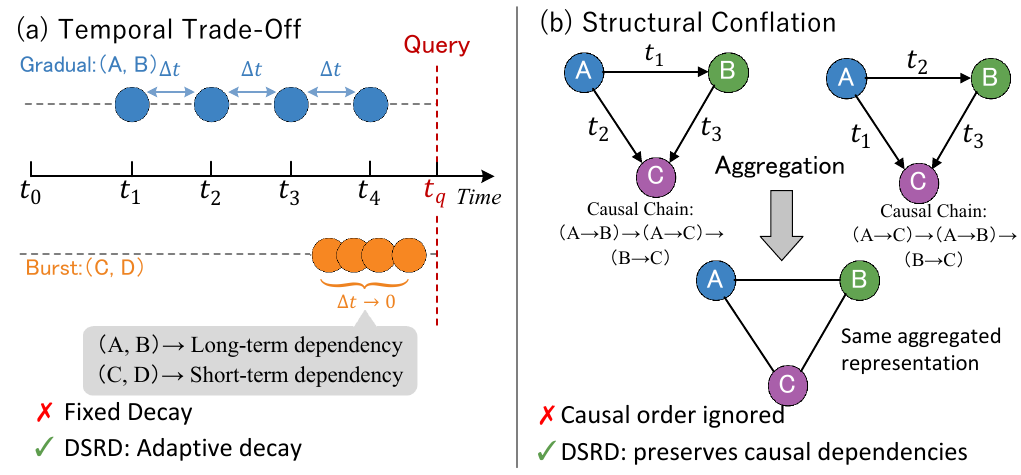}
    \caption{An illustration of the challenges in dynamic graph learning. (a) Temporal Trade-Off: Gradual $(A,B)$ and burst $(C,D)$ interaction patterns require distinct long-term and short-term dependencies, respectively, yet fixed decay cannot distinguish them. (b) Structural Conflation: Different causal orderings in temporal walks yield identical representations under standard aggregation. }
    \label{fig:motivation}
\end{figure}

Recent advances in dynamic graph learning span memory-based architectures~\cite{rossi2020tgn,kumar2019jodie}, attention-driven temporal encoders~\cite{yu2023dygformer,xu2020tgat}, and temporal walk-based structural models~\cite{wang2021cawn,lu2024tpnet}. While these approaches have achieved notable empirical success, they predominantly rely on a single temporal scale or a fixed structural propagation mechanism~\cite{nguyen2018continuous,holme2012temporal}. As a result, they struggle to adapt to heterogeneous dynamic regimes, where both temporal patterns and structural dependencies vary substantially across datasets (e.g., sparse vs. dense graphs) and downstream tasks (e.g., node classification vs. link prediction)~\cite{feng2025comprehensive}.

The first challenge is the temporal trade-off, arising from the coexistence of heterogeneous temporal regimes within and across dynamic graphs. For instance, as illustrated in Figure~\ref{fig:motivation}(a), gradual interaction patterns encode long-term relational dependencies~\cite{yu2023dygformer}, whereas bursty interactions reflect short-lived but informative events. Real-world dynamic graphs often exhibit both behaviors simultaneously, yet most existing models rely on fixed decay or uniform temporal encoding, preventing adaptive balancing between short-term and long-term dependencies~\cite{cong2023graphmixer,wang2021cawn,chung2025between}.

The second challenge is structural conflation, which concerns the loss of causal walks in temporal propagation. As shown in Figure~\ref{fig:motivation}(b), two walks with identical node sequences but different edge timestamps represent distinct information flow patterns (order of $A\rightarrow B$ and $A\rightarrow C$). However, conventional aggregation mechanisms treat them identically~\citep{lu2024tpnet}, collapsing high-order temporal dependencies that can be critical for downstream prediction~\cite{jin2022graph, ur2025primer}.

To address these challenges, we propose \textbf{D}ual-\textbf{S}cale \textbf{R}etentive \textbf{D}ynamics (\textbf{DSRD}), a unified framework that jointly models temporal and structural evolution within a shared  retentive state. Rather than treating temporal dependencies and structural propagation as independent modules, DSRD couples them through a shared state representation, allowing the model to learn graph-level trade-offs between short-term and long-term dependencies. Specifically, DSRD introduces two key components: (i) a unified retentive state with dual-scale adaptation that captures temporal dynamics through learnable decay kernels while simultaneously propagating signals along temporal walks for structural diffusion, and (ii) a gating mechanism with learnable time-sensitivity parameters that adaptively balances short-term injections and long-term retention based on the underlying interaction patterns. We provide theoretical analysis showing that the retentive dynamics admit a state-space formulation and satisfy stability bounds. Experiments on 14 benchmarks show that DSRD achieves state-of-the-art performance on dynamic link prediction and node classification under both transductive and inductive settings.

Our contributions are as follows:
\begin{itemize}
    \item We propose DSRD, a retentive framework that unifies temporal memory and structural diffusion through dual-scale adaptation, with learnable decay kernels that balance short-term and long-term dependencies.
    \item We provide supporting theoretical insights into the retentive state dynamics,including a closed-form historical expansion and boundedness properties that justify the stability of the learned representations under standard modeling assumptions.
    \item We conduct extensive experiments on 14 benchmark datasets, demonstrating consistent improvements over state-of-the-art methods under both transductive and inductive settings across discrete-time and continuous-time dynamic graphs.
\end{itemize}
\section{Preliminaries}
\label{sec:preliminaries}

\subsection{Problem Formulation}
\label{sec:problem}

\paragraph{Dynamic Graphs.}
We consider a dynamic graph
$\mathcal{G} = (\mathcal{V}, \mathcal{E}, X, \Phi)$,
where $\mathcal{V}$ denotes the node set and
$\mathcal{E} = \{(u,v,t)\}$ is a set of interaction events with timestamps $t \in \mathbb{R}_{\ge 0}$.
Each node $v \in \mathcal{V}$ is associated with a feature vector
$x_v \in X$, and each event $e=(u,v,t_e)$ carries an edge attribute
$\phi_e \in \Phi$.
Depending on the temporal granularity, dynamic graphs can be modeled in two principal paradigms: \emph{Discrete-Time Dynamic Graphs (DTDGs)}, which represent temporal evolution as a sequence of graph snapshots aggregated over fixed time intervals, and \emph{Continuous-Time Dynamic Graphs (CTDGs)}, which capture fine-grained dynamics through asynchronous event streams where each edge is associated with an exact timestamp.
At any reference time $t$, the observable snapshot is the causal subgraph
$\mathcal{G}_{t} = (\mathcal{V}, \mathcal{E}_{t}, X, \Phi_{t})$, where $\mathcal{E}_{t}=\{(u,v,t_e) \in \mathcal{E} \mid t_e \le t$\} .

\paragraph{Temporal Walks.}
High-order structural context can be characterized through \emph{temporal walks}, which are sequences of edges that respect the chronological ordering of interactions. To quantify the cumulative influence propagated along such walks, we introduce a diffusion score matrix $A_t \in \mathbb{R}^{|\mathcal{V}| \times |\mathcal{V}|}$, where entry $A_t(i,j)$ measures the aggregated weight of all time-respecting walks from node $i$ to node $j$ up to time $t$. Rather than serving as explicit node representations, these scores act as structural routing coefficients that govern how temporally localized signals propagate across the graph.

\paragraph{Learning Objective.}
Given a node $v$ at time $t$, the goal is to learn a representation
$
Z_v(t) = f_\Theta(v, \mathcal{G}_{t})
$ with all trainable parameters $\Theta$
that encodes all relevant structural and temporal information up to time $t$. The learned representation can be used for tasks such as \emph{link prediction}, where the model estimates the likelihood of a future interaction $(u,v,t^+)$, and \emph{node classification}, where the model predicts the ground-truth  label of $v$ at time $t^+$. Both transductive and inductive evaluation settings are considered.

\subsection{Retentive State}
\label{sec:retentive_state}

Given the causal subgraph $\mathcal{G}_{t}$, each node $j$ at reference time $t$ is associated with a feature vector $x_j(t) \in \mathbb{R}^{d_x}$. We apply shared linear projections
\begin{equation}
Q_t^{j} = x_j(t) W_Q,\;
K_t^{j} = x_j(t) W_K,\;
V_t^{j} = x_j(t) W_V,
\label{eq:qkv_ret}
\end{equation}
where $W_Q, W_K, W_V \in \mathbb{R}^{d_x \times d}$ are learnable parameters.

To any interaction event $(i, j, t_e) \in \mathcal{E}_{t}$, we associate the rank-one second-order increment $(K_{t}^i)^{\!\top} V_{t}^i \in \mathbb{R}^{d \times d}$. This outer-product form encodes pairwise feature interactions emitted by node $i$ at $t$ for both CTDGs and DTDGs, serving as the atomic unit of information injection. A straightforward approach aggregates all historical increments uniformly into the matrix
$
\mathcal{H}_t^j =
\sum_{(i,j,t_e) \in \mathcal{E}_{t}} (K_{t_e}^i)^{\!\top} V_{t_e}^i,
\label{eq:Ht_ret}
$
from which the node-level output is computed via query contraction
$
O_t^j = Q_t^j \mathcal{H}_t^j
$~\cite{zhai2021aft,peng2023rwkv,sun2023retnet, gu2024mamba}.
This formulation treats all past interactions in a flat and parallel manner, without modeling (i) how information is temporally retained and (ii) how it diffuses structurally. To disentangle these mechanisms, we first introduce weighted instantaneous injections and layer-wise temporal walks.

\paragraph{Instantaneous Injection.}
For the target node $j$, we define its instantaneous injection from an arbitrary source node $i$, at time $t$, as a weighted aggregation of source messages:
\begin{equation}
\Delta_t^i =
\sum_{(i,j,t_e) \in \mathcal{E}_{t}} \omega(t,(i,j,t_e))\, (K_{t_e}^i)^{\!\top} V_{t_e}^i,
\label{eq:Delta_ret}
\end{equation}
where $\omega(t,(i,j,t_e)) \ge 0$ is a causal weight that modulates the contribution of each historical interaction based on temporal proximity and semantic relevance.  This quantity represents the localized signal emitted by source node $i$.

\paragraph{Temporal Walk Diffusion.}
Let $T_t \in \mathbb{R}^{|\mathcal{V}| \times |\mathcal{V}|}$ be the instantaneous transition matrix induced by $\mathcal{E}_t$. Given a maximum diffusion depth $K$, we formalize recursive walk diffusion kernels $\{ A_t^{(\ell)} \}_{\ell = 0}^{K}$ as
\begin{equation}
A_t^{(\ell)} = A_{t^-}^{(\ell)} + A_{t^-}^{(\ell-1)} T_t,
\quad \text{for } t, \ell > 0,
\label{eq:A_rec_ret}
\end{equation}
with $A_t^{(0)} = I$ for all $t$, and $A_0^{(\ell)} = 0$ for $\ell > 0$. Here $t^-$ denotes the most recent event time right before $t$. Recall that $A_t^{(\ell)}(i,j)$ denotes the aggregated weight of length-$\ell$ time-respecting walks from $i$ to $j$ up to time $t$. To gain more insight into the walk matrix, we provide a closed-form expansion of $A_t^{(\ell)}$ in Appendix~\ref{proofs} (see Lemma~\ref{lem:walk_expansion}).

\paragraph{Retentive State Dynamics.}
To couple short-term injections with long-term structural diffusion, we define a
latent retentive state $S_t^{j,(\ell)} \in \mathbb{R}^{d\times d}$ for each node $j$, diffusion depth $\ell > 0$, and time $t > 0$:
\begin{equation}
S_t^{j,(\ell)} =
a_t S_{t^-}^{j,(\ell)}
+
b_t
\bar{\Delta}_{t}^{j,(\ell)}
\qquad
S_0^{j,(\ell)} = 0,
\label{eq:S_ret}
\end{equation}
where $a_t,b_t \in [0,1]$ control long-term retention and short-term injection
strength, respectively. Note that $\bar{\Delta}_{t}^{j,(\ell)} = 
\sum_{i\in\mathcal{V}}A_{t}^{(\ell)}(i,j)\Delta_{t}^{i}$.
Although the summation is written over all nodes, only temporally reachable nodes
contribute non-zero terms, since $A_t^{(\ell)}(i,j)=0$ whenever no temporal
walk of length $\ell$ exists from $i$ to $j$.

\begin{proposition}
\label{prop:unified_state_rw}
Let $\mathcal{T}_{t}$ denote the ordered set of event timestamps up to $t$.
For each node $j$ and diffusion depth $\ell$, the retentive state admits the
closed-form expansion
\begin{equation}
S_t^{j,(\ell)} =
\sum_{\tau\in\mathcal{T}_{t}}
\left(
\prod_{r\in\mathcal{T}_{(\tau,t]}} a_r
\right)
b_\tau
\bar{\Delta}_{t}^{j,(\ell)}
\label{eq:S_closed_ret}
\end{equation}
where $\mathcal{T}_{(\tau,t]}=\{r\in\mathcal{T}_{t} | \tau<r\leq t\}$.
\end{proposition}

The proof of Proposition~\ref{prop:unified_state_rw} is provided in Appendix~\ref{app:proof_prop}. Note that~\cref{eq:S_closed_ret} unifies temporal retention and high-order structural diffusion within a single recurrent state. For each diffusion depth $\ell$, the resulting readout $O_t^{j,(\ell)} = Q_t^j S_t^{j,(\ell)}$ captures a temporally filtered and structurally diffused summary of all historical second-order injections. 
\section{\ours}
\label{sec:method}

In this section, we instantiate the DSRD framework with two core mechanisms: adaptive temporal retention and structural diffusion (see Section~\ref{sec:dsrd}). We additionally describe auxiliary components including time encoding, feature embedding, and the overall block architecture (see Section~\ref{sec:net}).

\begin{figure*}[!t]
\centering
    \includegraphics[width=0.94\linewidth]{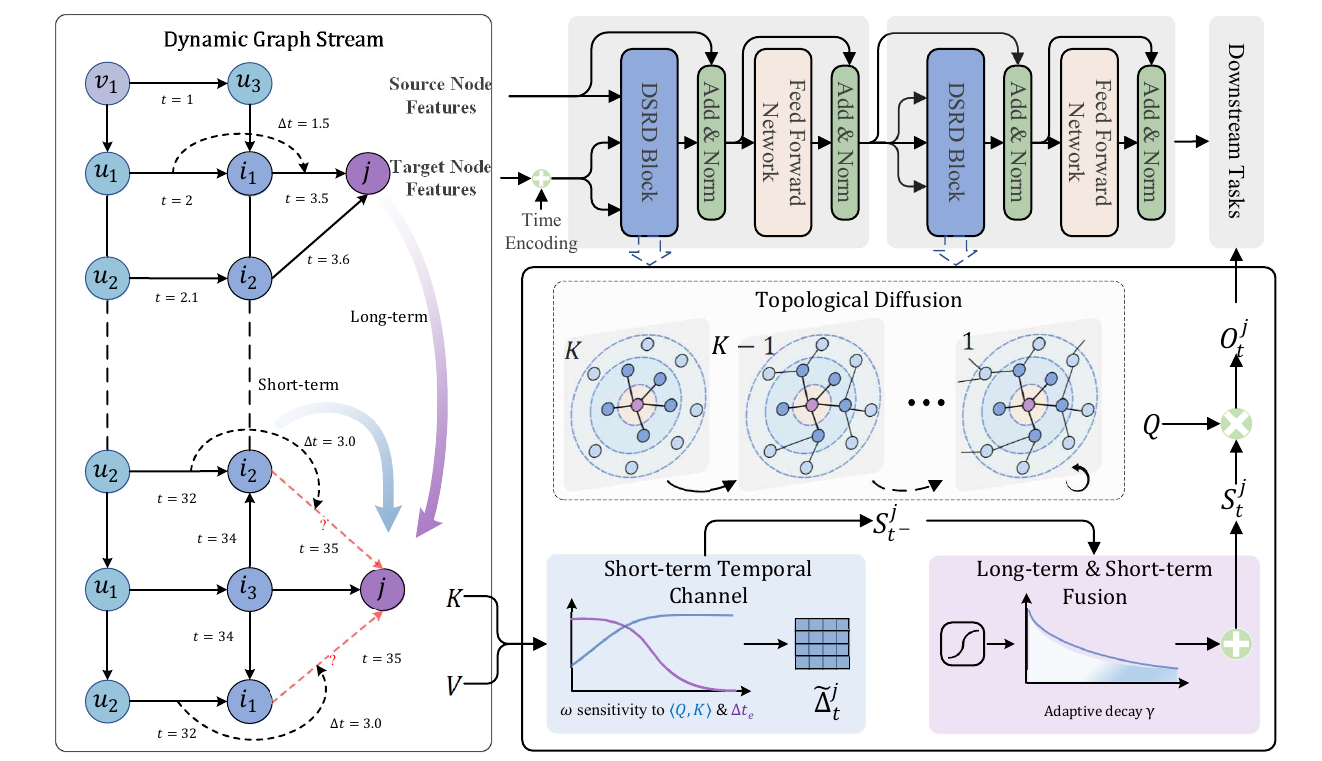}
    \caption{Overview of DSRD. 
The left panel shows a dynamic graph stream centered on target node $j$, where $i$, $u$, and $v$ denote one-hop, two-hop, and three-hop temporal neighbors. 
The lower-right part illustrates the three core operations of DSRD: short-term temporal injection, topological diffusion over time-respecting walks, and gated temporal (long-term vs. short-term) fusion. 
The upper-right part shows the stacked DSRD architecture with feed-forward layers for downstream tasks.}
    \label{fig:archi}
\end{figure*}

\subsection{Adaptive Dual-Scale Retentive Diffusion}
\label{sec:dsrd}

We model dynamic graph representation learning as a retentive diffusion process in which long-term temporal memory and high-order structural propagation are jointly maintained in a unified latent state as shown in~\cref{fig:archi}.
For each target node $j$ and diffusion layer $\ell$, the model maintains a second-order latent state
$S_{t}^{j,(\ell)} \in \mathbb{R}^{d \times d}$,
which is updated at any time $t$ when an interaction occurs.  Following the
unified retentive formulation established in Proposition~\ref{prop:unified_state_rw}, we introduce our strategies:

\paragraph{Short-Term Injection with Adaptive Decay.}
To separate instantaneous event responses from long-term structural retention, we instantiate the short-term injection $\Delta_t^i$ from Equation~\eqref{eq:Delta_ret} by specifying the causal weight as
\begin{equation}
\omega\big(t,(i,j,t_e)\big)
=
\underbrace{\sigma\big(\langle Q_t^j, K_{t_e}^i \rangle\big)}_{\text{query--key attention}}
\cdot
\underbrace{\exp\!\big(-\lambda(\Delta t_e)^{\alpha}\big)}_{\text{adaptive temporal decay}}.
\label{eq:short-decay}
\end{equation}
Here, the attention factor first computes cosine similarity $\langle Q_t^j, K_{t_e}^i \rangle \in (-1,1)$ between query and key vectors, then applies the sigmoid $\sigma(\cdot)$ to squash the result into $(0,1)$ as a context-aware gating coefficient. The second factor introduces adaptive temporal decay over the elapsed time $\Delta t_e = t - t_e$, where $\lambda > 0$ is a learnable decay rate and $\alpha \in (0,1)$ controls temporal sensitivity. This exponential forgetting enables adaptation to non-stationary dynamics.

\paragraph{Gated Long-term and Short-term Fusion.}
The short-term injection $\bar \Delta_t^{j,(\ell)}=\sum_{i \in \mathcal{V}}A_t^{(\ell)}(i,j)\Delta_t^i$ captures instantaneous event signals,
while the accumulated state $S_{t^-}^{j,(\ell)}$ preserves long-range historical
dependencies.
To balance these two components, we introduce a learnable gating mechanism.
Formally, we take $a_t = \gamma^{(\ell)}$ and $b_t = 1 - \gamma^{(\ell)}$ in Equation~\eqref{eq:S_ret}, where $\gamma^{(\ell)} \in (0,1)$ is a learnable gate, yielding the convex combination
\begin{equation}
S_t^{j,(\ell)}
=
\gamma^{(\ell)}\, S_{t^-}^{j,(\ell)}
+
(1-\gamma^{(\ell)})\, \bar \Delta_t^{j,(\ell)}.
\label{eq:retentive_update}
\end{equation}
A higher $\gamma^{(\ell)}$ retains more historical information, favoring long-term dependencies, while a lower $\gamma^{(\ell)}$ allows the model to respond more rapidly to recent events.

\begin{theorem}
\label{thm:stability}
Suppose the aggregated increments are uniformly bounded as $\|\sum_{(i,j,t_e)\in\mathcal{E}_t} \Delta_t^i\| \le M$ for all $t$, $j$, and $\ell$  where $\|\cdot\|$ denotes the Frobenius norm. Then:

(\romannumeral1) For all $t \ge 0$,
\begin{equation}
\|S_t^{j,(\ell)}\| \le \big(1 - (\gamma^ {(\ell)} )^{n}\big) M \le M,
\end{equation}
where $n = |\mathcal{T}_{t}|$ is the number of event timestamps up to $t$.

(\romannumeral2) The readout $O^{j,(\ell)}_t = Q^j_t S^{j,(\ell)}_t$ is uniformly bounded:
\begin{equation}
    \|O^{j,(\ell)}_t\| \leq LB_x \cdot M,
\end{equation}
where $L$ is the norm bound on the projection matrices and $B_x$ is the bound on node features.
\end{theorem}

We defer the proof of~\cref{thm:stability} to Appendix~\ref{app:proof_stability}. The boundedness assumption on aggregated increments is justified in Lemma~\ref{lem:bounded_increments} from Appendix~\ref{app:proof_stability}. 
Theorem~\ref{thm:stability} provides key guarantees for DSRD: (i) uniform boundedness prevents gradient explosion during training, (ii) output contraction ensures robustness to input perturbations.

\paragraph{Topological Diffusion via Temporal Walks.}
We note that the gated update in Equation~\eqref{eq:retentive_update} requires computing $\bar{\Delta}_t^{j,(\ell)} = \sum_{i \in \mathcal{V}} A_t^{(\ell)}(i,j) \Delta_t^i$, which involves the $\ell$-hop walk matrix $A_t^{(\ell)} \in \mathbb{R}^{|\mathcal{V}| \times |\mathcal{V}|}$. Explicitly maintaining this matrix for each layer $\ell$ incurs $\mathcal{O}(K|\mathcal{V}|^2)$ memory overhead, which is prohibitive for large-scale graphs. To circumvent this, we redefine the short-term injection as $\tilde{\Delta}_t^{j,(\ell)} = \sum_{(i,j,t_e) \in \mathcal{E}_t} \Delta_t^{i,(\ell)}$ with layer-specific projections, aggregating contributions from all incident edges at time $t$. Substituting into Equation~\eqref{eq:retentive_update}, the state update becomes:
\begin{equation}
S_t^{j,(\ell)}
=
\gamma^{(\ell)}\, S_{t^-}^{j,(\ell)}
+
(1-\gamma^{(\ell)})\, \tilde \Delta_t^{j,(\ell)}.
\label{eq:retentive_update_2}
\end{equation}

Finally, the readout for each layer is obtained via query contraction $O_t^{j,(\ell)} = Q_t^{j,(\ell)} S_t^{j,(\ell)}$, where we extend the projection defined in Equation~\eqref{eq:qkv_ret} to each layer $\ell$ by applying layer-specific parameters. The state is further updated by observing that the temporal walk matrix satisfies a recursive structure: when a new interaction $(i, j, t_e)$ arrives, existing $(\ell {-}1)$-step walks ending at $i$ extend to $\ell$-step walks ending at $j$. This motivates a layer-wise state propagation scheme inspired by~\citet{lu2024tpnet}, which implicitly encodes the walk matrix without explicit materialization (an illustrative example is provided in Appendix~\ref{appendix:sketch}). Specifically, we propagate information across walk lengths:
\begin{equation}
S_t^{j,(\ell)} \gets
S_{t}^{j,(\ell)}
+
\sum_{(i,j,t_e)\in\mathcal{E}_t}
\psi_{ij}^{(\ell)}(t_e)\, S_{t}^{i,(\ell-1)},\; \ell=K,\ldots,2,
\label{eq:interlayer_propagation}
\end{equation}
where $\psi_{ij}^{(\ell)}(t_e) = \exp(-\ell \delta^{(\ell)} \Delta t_e)$ is a trainable attenuation with a learnable parameter $\delta^{(\ell)}>0$.The readout precedes topological diffusion to preserve causality, while the short-term injection $\tilde{\Delta}_t^{j,(\ell)}$, depending only on immediate neighbors, remains unaffected, as shown in in~\cref{alg:dsrd}.

\begin{algorithm}[!tb]
  \caption{Dual-Scale Retentive Diffusion for DSRD}
  \label{alg:dsrd}
  \begin{algorithmic}
    \STATE {\bfseries Input:} Dynamic subgraph $\mathcal{G}_{ t} = (\mathcal{V}, \mathcal{E}_{ t})$, walk depth $K$.
    \STATE Initialize retentive states $S_0^{j,(\ell)}$ for all $j \in \mathcal{V}$, $\ell \in \{1, \ldots, K\}$.
    \FOR{each event $(i, j, t_e) \in \mathcal{E}_{ t}$}
      \FOR{$\ell = 1$ \textbf{to} $K$}
          \STATE Compute projections via Eq.~\eqref{eq:qkv_ret}.
          \STATE Compute short-term injection $\Delta_t^{i,(\ell)}$ via Eq.~\eqref{eq:Delta_ret}.
        \STATE Fuse short-term and long-term signals via Eq.~\eqref{eq:retentive_update_2}.
      \ENDFOR
      \STATE Obtain output $O_t^{j,(K)} = Q_t^j S_t^{j,(K)}$.
        \FOR{$\ell = K$ \textbf{down to} $2$}
        \STATE Update topological diffusion via Eq.~\eqref{eq:interlayer_propagation}.
      \ENDFOR
    \ENDFOR
    \STATE {\bfseries Output:} Final node representations $\{O_t^{j,(K)}\}_{j \in \mathcal{V}}$
  \end{algorithmic}
\end{algorithm}

\subsection{Network Composition and Optimization Objective}
\label{sec:net}
\paragraph{Auxiliary Augmentation.}
To model both transient interactions and long-term dependencies in dynamic graphs, \our{} is organized as a hierarchical composition of model layers. Specifically,  for each event $(u, v, t_e)$, we compute a time encoding~\cite{cong2023graphmixer} $\phi_t = \cos(\Delta t_e\, W_{t_e})$ and fuse it with the edge attribute to obtain $\tilde{\phi}_e = (\phi_e + \phi_t) W_e$, where $W_{t_e}$ and $W_e$ are trainable parameters. The temporal-aware edge features are injected into node representations by augmenting the projections:
\begin{equation}
\tilde{K}_{t_e}^i = K_{t_e}^i + \tilde{\phi}_e W_K, \quad \tilde{V}_{t_e}^i = V_{t_e}^i + \tilde{\phi}_e W_V.
\end{equation}
These augmented projections replace $(K_{t_e}^i, V_{t_e}^i)$ in Equation~\eqref{eq:Delta_ret}. As a result, the increment $(\tilde{K}_{t_e}^i)^\top \tilde{V}_{t_e}^i$ encodes not only the node features but also edge-specific temporal signals, enabling the model to differentiate between interactions that share the same target nodes but occur at different times or carry distinct attributes.

\paragraph{Model Architecture.} 
Within the $\ell$-th layer, we extend the retentive state to $H$ parallel attention heads, where each head $h \in \{1,\ldots,H\}$ maintains independent parameters $(\omega^{h,(\ell)}, \gamma^{h,(\ell)}, \psi^{h,(\ell)})$ as defined in Section~\ref{sec:dsrd}. The node representation $Z^{(\ell)}$ for each node $j$ is updated as
\begin{equation}
Z^{(\ell)} = X^{(\ell)} 
+ \mathrm{LN}\!\Big(\mathrm{Concat}(\, g_\ell^h(X^{(\ell)},S^{j,h,(\ell)}))\Big),
\end{equation}
where $g_\ell^h(\cdot)$ denotes the \our{} operator for head $h$, $\mathrm{Concat}$ denotes concatenation across heads, and $\mathrm{LN}$ represents layer normalization~\cite{ba2016layernrom} applied across the concatenated feature channels.
Following the core \our{} operation, a feed-forward sublayer~\cite{vaswani2017attention} refines the representation as
\begin{equation}
X^{(\ell+1)} =
Z^{(\ell)}
+
\mathrm{GELU}\!\big(\mathrm{LN}(Z^{(\ell)})\,W_1^{(\ell)}\big)\,W_2^{(\ell)},
\end{equation}
where $\mathrm{GELU}$ denotes the Gaussian Error Linear Unit activation~\cite{hendrycks2016GELU} and $W_1^{(\ell)},W_2^{(\ell)}$ are learnable parameters. 
The alternating application of multi-head retentive propagation and feed-forward refinement enables each layer to encode both fine-grained temporal signals and higher-order structural context.

\paragraph{Normalization.}
To stabilize the numerical flow, we apply several normalization techniques that preserve scale-invariant properties: \romannumeral1) The readout $O_t^{j,{(\ell)}}$ for each head $h$ is scaled as $Q_t^{j,{(\ell)}}S_t^{j,{(\ell)}} / \sqrt{d}$; \romannumeral2) The structural propagation weights $\psi_{ij}^{(\ell)}(t)$ and temporal attention weights $\omega_{ij}$ are each normalized over the corresponding neighborhood as $\tilde{\psi}_{ij}^{(\ell)}(t) = \psi_{ij}^{(\ell)}(t) / \max(\sum_{i'} \psi_{i'j}^{(\ell)}(t), 1)$ and $\tilde{\omega}_{ij} = \omega_{ij} / \max(\sum_{i'} |\omega_{i'j}|, 1)$; \romannumeral3) All time intervals $\Delta t$ involved in decay functions are transformed using $\log(1 + \Delta t)$ prior to exponentiation, improving numerical stability across various time scales (from days to years).

After stacking $K$ \our{} layers, the final node embedding is obtained by mean pooling over the $H$ heads, $h_t^{u}=\mathrm{MEAN}_{h=1}^{H}(O_t^{u,h,(K)})$. 
For link prediction, the likelihood of a future interaction $(u,v,t^+)$ is computed as 
$p_t^{u,v}=\mathrm{MLP}([h_t^u,h_t^v])$. 
For node classification, the label probability is given by $\hat{y}_t^u = \sigma(h_t^u W_{\mathrm{cls}})$, where $W_{\mathrm{cls}}$ is a learnable weight matrix.
\section{Experiments}
\label{section-5}
In this section, we conduct extensive experiments on real-world datasets with diverse downstream tasks and settings to evaluate \our{} on effectiveness and scalability.

% requires \usepackage[table]{xcolor} in preamble
\begin{table*}[!tb]
\centering
\definecolor{mygreen}{RGB}{122,197,180}
\caption{AP for transductive dynamic link prediction with random (rnd), historical (hist), and inductive (ind) negative sampling strategies (NSS). Avg. Rank is computed across datasets for each NSS. The best model is shown in \textbf{bold} and the second best model is \underline{underlined}. \colorbox{mygreen!50}{Shaded cells} indicate relative performance across top 50\% models on a given dataset.}
\label{tab:transductive_ap}
\resizebox{1.01\textwidth}{!}
{
\setlength{\tabcolsep}{0.9mm}
{
\begin{tabular}{c|c|ccccccccccccc}
\toprule
NSS & Datasets & JODIE & DyRep & TGAT & TGN & CAWN & EdgeBank & TCL & GraphMixer & PINT & NAT & DyGFormer & TPNet & DSRD \\
\midrule
\multirow{14}{*}{rnd} & Wikipedia & 94.12{\scriptsize$\pm$0.68} & 93.20{\scriptsize$\pm$0.41} & 91.39{\scriptsize$\pm$0.11} & 96.08{\scriptsize$\pm$0.03} & \cellcolor{mygreen!17}98.21{\scriptsize$\pm$0.01} & 91.10{\scriptsize$\pm$0.00} & 91.98{\scriptsize$\pm$0.46} & 96.15{\scriptsize$\pm$0.07} & \cellcolor{mygreen!18}98.22{\scriptsize$\pm$0.04} & 97.89{\scriptsize$\pm$0.08} & \cellcolor{mygreen!36}98.57{\scriptsize$\pm$0.03} & \underline{\cellcolor{mygreen!67}99.14{\scriptsize$\pm$0.00}} & \textbf{\cellcolor{mygreen!100}99.76{\scriptsize$\pm$0.07}} \\
 & Reddit & 98.23{\scriptsize$\pm$0.08} & 98.16{\scriptsize$\pm$0.05} & 96.02{\scriptsize$\pm$0.29} & 98.63{\scriptsize$\pm$0.02} & 98.70{\scriptsize$\pm$0.03} & 95.49{\scriptsize$\pm$0.00} & 90.71{\scriptsize$\pm$0.27} & 96.81{\scriptsize$\pm$0.01} & \cellcolor{mygreen!36}99.17{\scriptsize$\pm$0.02} & \cellcolor{mygreen!20}98.95{\scriptsize$\pm$0.11} & \cellcolor{mygreen!7}98.79{\scriptsize$\pm$0.00} & \underline{\cellcolor{mygreen!51}99.35{\scriptsize$\pm$0.01}} & \textbf{\cellcolor{mygreen!100}99.98{\scriptsize$\pm$0.00}} \\
 & MOOC & 82.36{\scriptsize$\pm$0.60} & 79.97{\scriptsize$\pm$1.69} & 80.23{\scriptsize$\pm$0.61} & \cellcolor{mygreen!37}88.62{\scriptsize$\pm$0.67} & 76.10{\scriptsize$\pm$0.10} & 59.35{\scriptsize$\pm$0.00} & 79.09{\scriptsize$\pm$0.01} & 81.35{\scriptsize$\pm$0.07} & \cellcolor{mygreen!36}88.40{\scriptsize$\pm$1.57} & \cellcolor{mygreen!22}86.04{\scriptsize$\pm$0.73} & 79.66{\scriptsize$\pm$0.97} & \underline{\cellcolor{mygreen!61}92.66{\scriptsize$\pm$0.64}} & \textbf{\cellcolor{mygreen!100}99.25{\scriptsize$\pm$0.18}} \\
\bottomrule
\end{tabular}
}
}
\end{table*}

\subsection{Setup}
\label{section-5:setup}

\textbf{Datasets and Baselines.}
We conduct experiments on fourteen benchmark datasets sourced from real-world dynamic systems collected by \citet{poursafaei2022towards}, including Wikipedia, Reddit, MOOC, Myket, LastFM, Enron, Social Evo., UCI, Flights, Can. Parl., USLegis., UN Trade, UN Vote, and Contact. Since the original node classification datasets are
extremely imbalanced, we further include PubMed as a
less severely imbalanced biomedical dataset. Specifically, the
positive-class ratios are only 0.14\% on Wikipedia, 0.05\%
on Reddit, and 0.99\% on MOOC, while PubMed~\cite{ara2024pubmed} contains
19,732 (22.26\%) and 68,916 (77.74\%) instances in the
two classes. Twelve strong dynamic graph learning methods are included as baselines, i.e., JODIE~\cite{kumar2019jodie}, DyRep~\cite{trivedi2019dyrep}, TGAT~\cite{xu2020tgat}, TGN~\cite{rossi2020tgn}, CAWN~\cite{wang2021cawn}, EdgeBank~\cite{poursafaei2022towards}, TCL~\cite{wang2021tcl}, GraphMixer~\cite{cong2023graphmixer}, NAT~\cite{luo2022nat}, PINT~\cite{souza2022pint}, DyGFormer~\cite{yu2023dygformer}, and TPNet~\cite{lu2024tpnet}. Additional details about datasets and baselines are provided in the Appendix~\ref{appendix:datasets} and~\ref{apppendix:baselines}.

\begin{table*}[!ht]
\centering
\definecolor{mygreen}{RGB}{122,197,180}
\caption{ROC-AUC for node classification. Avg. Rank is computed across datasets. \colorbox{mygreen!50}{Shaded cells} indicate relative performance across top 50\% models on a given dataset.}
\label{tab:node_class_test_auc}
\resizebox{1.01\textwidth}{!}
{
\setlength{\tabcolsep}{0.9mm}
{
\begin{tabular}{c|cccccccccccc}
\toprule
Dataset & JODIE & DyRep & TGAT & TGN & CAWN & TCL & GraphMixer & PINT & NAT & DyGFormer & TPNet & DSRD \\
\midrule
Wikipedia & \cellcolor{mygreen!42}86.11{\scriptsize$\pm$1.21} & \textbf{\cellcolor{mygreen!100}88.06{\scriptsize$\pm$1.35}} & 69.78{\scriptsize$\pm$0.73} & \cellcolor{mygreen!57}86.63{\scriptsize$\pm$0.38} & 80.03{\scriptsize$\pm$1.74} & 73.59{\scriptsize$\pm$5.49} & 84.70{\scriptsize$\pm$0.16} & \cellcolor{mygreen!47}86.27{\scriptsize$\pm$0.26} & 64.98{\scriptsize$\pm$1.39} & 84.35{\scriptsize$\pm$1.27} & 83.06{\scriptsize$\pm$0.69} & \underline{\cellcolor{mygreen!90}87.72{\scriptsize$\pm$1.14}} \\
Reddit & 61.34{\scriptsize$\pm$2.31} & 60.94{\scriptsize$\pm$2.10} & 58.85{\scriptsize$\pm$2.44} & \cellcolor{mygreen!39}63.98{\scriptsize$\pm$2.48} & 62.31{\scriptsize$\pm$0.95} & 61.23{\scriptsize$\pm$0.72} & \underline{\cellcolor{mygreen!100}65.72{\scriptsize$\pm$2.73}} & \cellcolor{mygreen!71}64.90{\scriptsize$\pm$0.34} & 58.07{\scriptsize$\pm$1.85} & \cellcolor{mygreen!70}64.86{\scriptsize$\pm$0.97} & \textbf{\cellcolor{mygreen!100}65.72{\scriptsize$\pm$1.09}} & 62.85{\scriptsize$\pm$2.15} \\
MOOC & \cellcolor{mygreen!1}68.80{\scriptsize$\pm$0.77} & 68.74{\scriptsize$\pm$1.76} & 63.89{\scriptsize$\pm$0.61} & \cellcolor{mygreen!5}69.24{\scriptsize$\pm$0.55} & 64.88{\scriptsize$\pm$0.81} & \cellcolor{mygreen!14}70.14{\scriptsize$\pm$1.22} & 67.02{\scriptsize$\pm$0.56} & 68.03{\scriptsize$\pm$0.32} & 55.86{\scriptsize$\pm$5.48} & \underline{\cellcolor{mygreen!73}76.04{\scriptsize$\pm$0.70}} & 64.91{\scriptsize$\pm$2.63} & \textbf{\cellcolor{mygreen!100}78.68{\scriptsize$\pm$0.53}} \\
PubMed & 54.59{\scriptsize$\pm$0.88} & 54.24{\scriptsize$\pm$0.36} & 52.69{\scriptsize$\pm$1.95} & \cellcolor{mygreen!5}57.31{\scriptsize$\pm$0.48} & 54.61{\scriptsize$\pm$1.91} & \cellcolor{mygreen!12}58.79{\scriptsize$\pm$0.50} & 55.25{\scriptsize$\pm$0.26} & 56.32{\scriptsize$\pm$0.96} & \underline{\cellcolor{mygreen!77}71.92{\scriptsize$\pm$3.21}} & \cellcolor{mygreen!20}60.29{\scriptsize$\pm$0.39} & 54.40{\scriptsize$\pm$0.56} & \textbf{\cellcolor{mygreen!100}76.46{\scriptsize$\pm$1.03}} \\
\midrule
Avg. Rank & 6.75 & 7.00 & 11.25 & 4.25 & 8.50 & 6.50 & 5.62 & 5.00 & 9.50 & \underline{4.00} & 7.12 & \textbf{2.50} \\
\bottomrule
\end{tabular}
}
}
\end{table*}

\textbf{Evaluation Settings.}
As stated in Section~\ref{sec:problem}, we evaluate all baselines and our method on dynamic link prediction and dynamic node classification tasks, and the link prediction task is further assessed under both transductive and inductive settings in order to capture performance across different generalization regimes. Three negative sampling strategies (i.e., random, historical and inductive negative sampling strategies) are used for evaluation, where we employ the random negative sampling strategy for model training. For evaluation metrics, Average Precision (AP) is used to measure predictive precision, while the Area Under the Receiver Operating Characteristic Curve (ROC-AUC) provides an estimate of sensitivity and robustness for both link prediction and node classification.

\textbf{Implementation and Configurations.}
We follow the standard training pipeline and evaluation protocol proposed by the Dynamic Graph Library (DyGLib)\footnote{\url{https://github.com/yule-BUAA/DyGLib}} for the training and evaluation of all models. To ensure fair comparison, all methods are trained and evaluated under identical pipelines. All experiments are conducted on a computing platform running Ubuntu 24.04.3 LTS, equipped with an Intel(R) Xeon(R) Gold 6548Y+ CPU @ 2.50 GHz and two NVIDIA L40S 48Q GPUs, while the software environment consists of Python 3.12, CUDA 13.0, and PyTorch 2.6.0. More details about the implementation and configurations are provided in Appendix~\ref{appendix:configurations}

\subsection{Main Experimental Results}
\label{section-5:comparison}

Across all benchmarks and evaluation protocols, the proposed method achieves the strongest overall performance (see statistical tests in Appendix~\ref{appendix:stats_rank}), attaining outstanding average ranks on dynamic link prediction (Tables~\ref{tab:transductive_ap},~\ref{tab:transductive_auc},~\ref{tab:inductive_ap}, and~\ref{tab:inductive_auc}) and dynamic node classification tasks (Table~\ref{tab:node_class_test_auc}), which establishes its overall effectiveness and scalability.
We further observe the following:

\textbf{(i) Generalization.} Across different evaluation settings, including multiple negative sampling strategies as well as transductive and inductive splits, performance remains consistently strong. In particular, generalization to unseen nodes under inductive evaluation is stable, whereas several baseline methods exhibit noticeable degradation under the same conditions (see Tables~\ref{tab:inductive_ap} and~\ref{tab:inductive_auc}).
\textbf{(ii) Downstream tasks.} Across downstream tasks, competitive results are observed for both link prediction and node classification. Link prediction performance remains top-ranked across datasets and settings, while node classification results achieve the best average rank, with the highest ROC-AUC attained on MOOC and competitive performance maintained under highly imbalanced label distributions (see Appendix~\ref{appendix:node_calssificaiton}).
\textbf{(iii) Dataset regimes.} Across datasets with diverse temporal and structural characteristics (see Table~\ref{tab:datasets}), consistent performance trends emerge. Clear advantages are observed on discrete-time dynamic graphs with coarse temporal granularity, competitive results are maintained on continuous-time graphs with fine-grained timestamps, and strong performance is achieved on high-density graphs with rich structural connectivity (see Sections~\ref{sec:adaptive-decay} and~\ref{sec:ablation}).
\textbf{(iv) Scalability.} In terms of scalability, favorable efficiency characteristics are observed. Running latency and throughput remain competitive with existing methods while using fewer parameters, indicating that the approach scales effectively to large dynamic graphs (see Appendix~\ref{appendix:scalability}).

We also observe that several methods achieve near-saturated AP and ROC-AUC scores on a subset of benchmarks. Notably, such near-saturation arises only under specific evaluation settings and does not persist consistently across different negative sampling strategies or tasks, with systematic and interpretable variations observed in dynamic link prediction, as further discussed in Appendix~\ref{appendix:saturated_performance}.

\begin{figure}
    \centering
    \includegraphics[width=\linewidth]{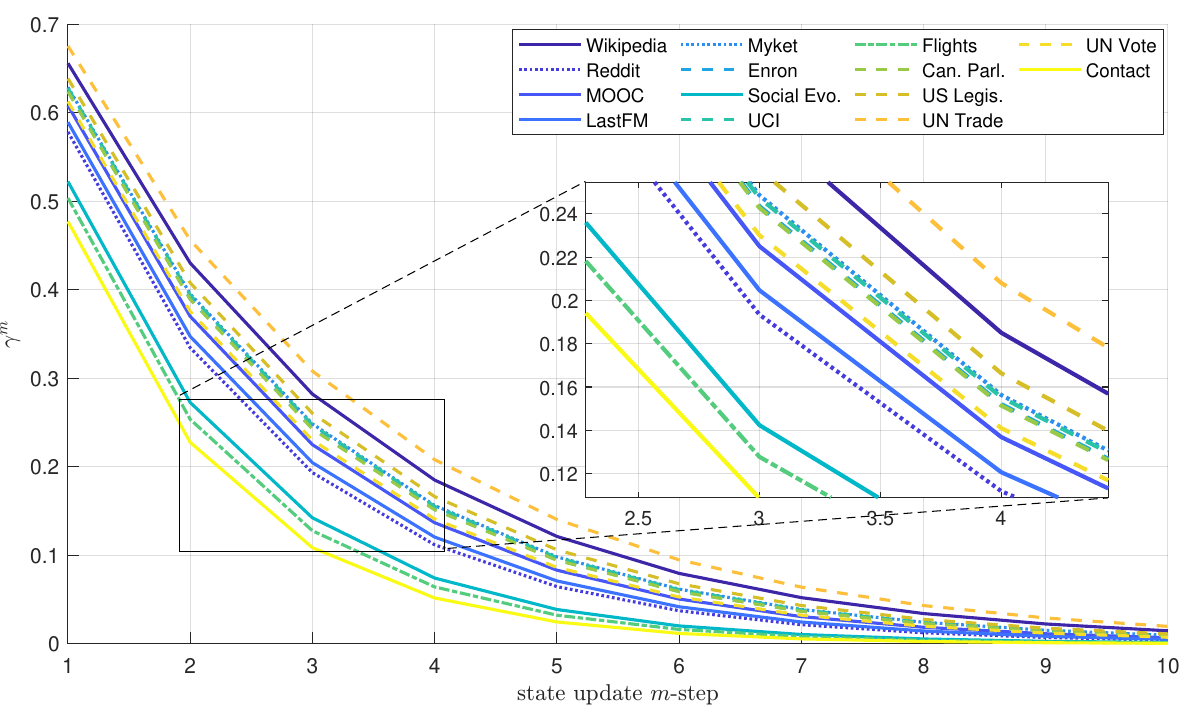}
    \caption{Cumulative long-term retention $\gamma^m$ versus update steps $m$ for all datasets, where $\gamma = \gamma^{(1)}$ is the learned retention gate from the first layer of the trained model.}
    \label{fig:long}
\end{figure}

\subsection{How do dual-scale adaptive decays reveal the temporal and structural effects?}
\label{sec:adaptive-decay}
To examine whether the proposed joint decay mechanism can adaptively align with the diverse temporal and structural dynamics, we analyze the learned decay behaviors across all datasets. Specifically, we extract three types of decay patterns from the trained models, including long-term dependency decay $\gamma$ governed by state update steps $m$ (see Figure~\ref{fig:long}), short-term temporal decay $\omega$ shaped by time interval sensitivity (see Figure~\ref{fig:short}), and topological decay $\psi$ reflecting hop-level topological attenuation (see Figure~\ref{fig:topo}).

The learned parameters exhibit dataset-specific separation that partially correlates with graph properties in Table~\ref{tab:datasets}, though domain characteristics also influence the patterns. Discrete-time graphs with high interaction scale (e.g., UN Trade, UN Vote, Can. Parl.) learn strong long-term retention, preserving information across sparse temporal snapshots. Dense proximity graphs (Contact, Social Evo.) show rapid state turnover but broad receptive fields, whereas sparse bipartite graphs (Reddit, Myket) exhibit high sensitivity to temporal recency and structural locality.

\subsection{Ablation Study}
\label{sec:ablation}
We conduct an ablation study to assess the contribution of individual components in DSRD. All variants are evaluated under identical data splits, training settings, and evaluation protocols as the full model. The following ablated variants are considered:
\textit{\romannumeral1)} w/o DSRD Block, which removes the entire DSRD module (Eqs.~\ref{eq:Delta_ret},~\ref{eq:interlayer_propagation}, and~\ref{eq:retentive_update_2});
\textit{\romannumeral2)} w/o Temporal Decay, which disables the adaptive temporal decay in short-term injection (Eq.~\ref{eq:short-decay});
\textit{\romannumeral3)} w/o Topological Diffusion, which restricts propagation to single-hop neighborhoods (Eq.~\ref{eq:interlayer_propagation});
\textit{\romannumeral4)} w/o State, which removes the retentive state update (Eqs.~\ref{eq:interlayer_propagation} and~\ref{eq:retentive_update_2}).
\begin{figure}[!ht]
    \centering
    \includegraphics[width=\linewidth]{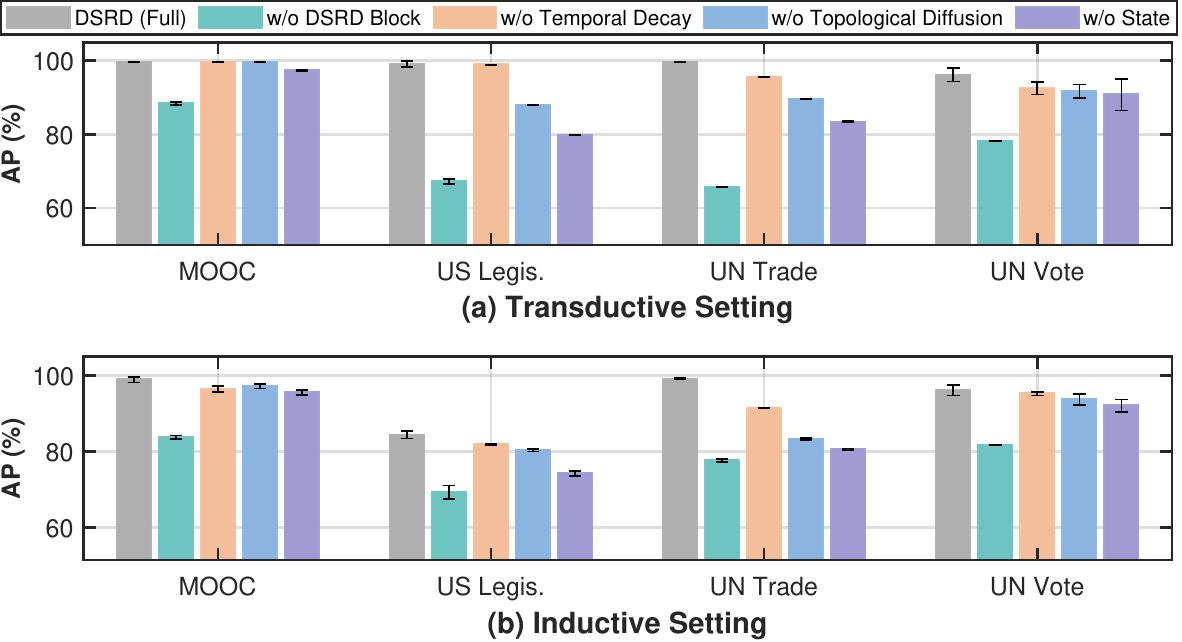}
    \caption{Partial ablation results of DSRD under (a) transductive and (b) inductive settings.}
    \label{fig:ablation-AP-Part}
\end{figure}
Figure~\ref{fig:ablation-AP-Part} reports the results under both transductive and inductive settings. Across datasets, removing any component leads to  performance degradation compared to the full model. The impact of each variant is observable under both transductive and inductive settings, with more pronounced drops on datasets exhibiting denser structure or longer temporal spans. Additional ablation results and extended comparisons are provided in the Appendix~\ref{appendix:ablation}.

\subsection{Time Complexity Analysis}

Let $\kappa$ denote the number of sampled temporal neighbors per node, $K$ the number of DSRD layers for multi-hop structural propagation, and $d_h = d/H$ the per-head dimension. For each event, DSRD performs neighbor sampling, second-order state injection, and layered topological diffusion. The per-event computation involves: (i) linear projections for $Q$ and $K, V$ with cost $\mathcal{O}(\kappa d_x d)$, (ii) second-order outer-product injection $K^\top V$ with cost $\mathcal{O}(\kappa H d_h^2)$ over sampled edges, and (iii) gated retentive state update and query contraction with cost $\mathcal{O}(H d_h^2)$ per node. Since topological diffusion is realized through $K$ stacked DSRD layers rather than explicit $K$-hop neighbor expansion, the total per-event cost is $\mathcal{O}(K \cdot (\kappa d_x d + \kappa d^2/H))$. When $d_x \approx d$, this simplifies to $\mathcal{O}(K \cdot \kappa \cdot d^2)$. Over the entire event stream, the time complexity is
$
\mathcal{O}(|\mathcal{E}_{ t}| \cdot K \cdot \kappa \cdot d^2).
$
Note that the complexity scales linearly with the number of events $|\mathcal{E}_{t}|$ and remains independent of $|\mathcal{V}|$.

\section{Related Work}
\label{section-2}

In this section, we give a structured review of dynamic graph learning methods and provide additional discussion and references in Appendix~\ref{appendix:related-work}.

\subsection{Long-Term Dependency and Memory Mechanisms}
Dynamic graph models often incorporate memory modules or recurrent architectures to capture long-term dependencies~\cite{feng2025comprehensive,gravina2024deep,chang2026grn}. Early approaches like JODIE maintain coupled recurrent states for user–item interactions, updating each entity’s embedding with every event and even projecting future embeddings along a trajectory~\cite{kumar2019jodie}. Similarly, DyRep models communication events via a temporal point process, using event-driven message passing to evolve node states over continuous time~\cite{trivedi2019dyrep}. Temporal Graph Networks (TGN) generalize this idea by assigning each node a persistent memory vector that is updated upon neighbor interactions, thereby encoding both short-term and long-term context in the node’s history~\cite{rossi2020tgn}. These memory-based frameworks enable retention of past influences beyond immediate neighbors, but they typically operate with a single fixed update scheme or decay rate. As a result, they struggle to adapt to heterogeneous dynamic regimes where different relationships exhibit vastly different temporal patterns~\cite{cong2023graphmixer}. For example, interactions that are infrequent yet persistent require long-term memory, whereas bursty interactions demand rapid forgetting, a balance that fixed-decay memories cannot achieve. This limitation often forces a trade-off: a model tuned for one timescale may underperform when the temporal dynamics shift, highlighting the need for more adaptive memory mechanisms in dynamic graph learning~\cite{lu2024tpnet}.
\section{Conclusion}
\label{section-6}

In this paper, we propose a dual-scale retentive dynamics framework that couples temporal memory and structural diffusion within a unified state representation. By introducing learnable decay kernels and gated memory updates, DSRD adaptively balances short-term responsiveness and long-term retention. Theoretical analysis establishes stability guarantees for the learned dynamics. Experiments on 14 benchmarks demonstrate state-of-the-art performance across diverse temporal regimes and evaluation settings. We discuss limitations and future directions in Appendix~\ref{appendix:concerns}.

% \section*{Impact Statement}
% This work advances dynamic graph representation learning, with potential applications in recommendation systems, social network analysis, and epidemic forecasting. While improved temporal link prediction can enhance user experience and enable early warning systems, such capabilities also raise concerns about privacy and potential misuse in surveillance or behavioral tracking. We encourage practitioners to consider ethical guidelines and privacy-preserving techniques when deploying dynamic graph models in sensitive applications. This paper presents foundational methodological research, and we do not anticipate immediate negative societal consequences from the work itself.

\bibliography{references}
\bibliographystyle{icml2026}

%%%%%%%%%%%%%%%%%%%%%%%%%%%%%%%%%%%%%%%%%%%%%%%%%%%%%%%%%%%%%%%%%%%%%%%%%%%%%%%
%%%%%%%%%%%%%%%%%%%%%%%%%%%%%%%%%%%%%%%%%%%%%%%%%%%%%%%%%%%%%%%%%%%%%%%%%%%%%%%
% APPENDIX
%%%%%%%%%%%%%%%%%%%%%%%%%%%%%%%%%%%%%%%%%%%%%%%%%%%%%%%%%%%%%%%%%%%%%%%%%%%%%%%
%%%%%%%%%%%%%%%%%%%%%%%%%%%%%%%%%%%%%%%%%%%%%%%%%%%%%%%%%%%%%%%%%%%%%%%%%%%%%%%
\newpage

\appendix
\onecolumn
\label{section-appendix}

\section{Proofs}
\label{proofs}

\subsection{Proof of Proposition~\ref{prop:unified_state_rw}}
\label{app:proof_prop}
We first establish a key lemma characterizing the walk diffusion kernel.

\begin{lemma}
\label{lem:walk_expansion}
Let $\mathcal{T}_{t}$ denote the ordered set of event timestamps up to $t$. We adopt the convention that the entries of the set $\mathcal{T}_{t}$ are $\tau_1<\tau_2<\cdots<\tau_n$. For $\ell, t > 0$, we have:
\begin{equation}
A_t^{(\ell)} = 
\begin{cases}
\displaystyle\sum_{\substack{\tau_1, \ldots, \tau_\ell \in \mathcal{T}_{t} \\ \tau_1 < \tau_2 < \cdots < \tau_\ell}} T_{\tau_1} T_{\tau_2} \cdots T_{\tau_\ell}, & \text{if } |\mathcal{T}_{t}| \geq \ell, \\[2ex]
0, & \text{otherwise}.
\end{cases}
\label{eq:walk_expansion_lemma}
\end{equation}
\end{lemma}

\begin{proof}[Proof of Lemma~\ref{lem:walk_expansion}]
For ease of calculation, we first show that the following identity holds for all $\ell, t > 0$:
\begin{equation}
A_t^{(\ell)} = \sum_{\tau \in \mathcal{T}_{t}} A_{\tau^-}^{(\ell-1)} T_\tau.
\label{eq:A_sum_identity}
\end{equation}

\textit{Base case: $t = \tau_1$.}
For $\ell = 1$, from the recursion in Equation~\eqref{eq:A_rec_ret}, we have
\begin{equation}
A_{\tau_1}^{(1)} = A_0^{(1)} + A_0^{(0)} T_{\tau_1} = 0 + I \cdot T_{\tau_1} = T_{\tau_1}.
\end{equation}
Similarly, the right-hand side of Equation~\eqref{eq:A_sum_identity} gives:
\begin{equation}
\sum_{\tau \in \mathcal{T}_{\tau_1}} A_{\tau^-}^{(0)} T_\tau = A_0^{(0)} T_{\tau_1} = T_{\tau_1},
\end{equation}
which confirms that Equation~\eqref{eq:A_sum_identity} holds for $\ell = 1$ at time $\tau_1$.

For any $\ell \geq 2$, we have
\begin{equation}
A_{\tau_1}^{(\ell)} = A_0^{(\ell)} + A_0^{(\ell-1)} T_{\tau_1} = 0 + 0 \cdot T_{\tau_1} = 0,
\end{equation}
and
\begin{equation}
\sum_{\tau \in \mathcal{T}_{\tau_1}} A_{\tau^-}^{(\ell-1)} T_\tau = A_0^{(\ell-1)} T_{\tau_1} = 0,
\end{equation}
confirming Equation~\eqref{eq:A_sum_identity} for all $\ell \geq 1$ at time $\tau_1$.

\textit{Inductive step.}
Assume that Equation~\eqref{eq:A_sum_identity} holds for a fixed $\ell$ at time $\tau_{q-1}$ where $q \in \{2, \ldots, n\}$. We demonstrate that it also holds at time $\tau_q$.

Using the recursion in Equation~\eqref{eq:A_rec_ret}, we have
\begin{align*}
A_{\tau_q}^{(\ell)} 
&= A_{\tau_{q-1}}^{(\ell)} + A_{\tau_{q-1}}^{(\ell-1)} T_{\tau_q} \\
&= \sum_{\tau \in \mathcal{T}_{\tau_{q-1}}} A_{\tau^-}^{(\ell-1)} T_\tau + A_{\tau_{q-1}}^{(\ell-1)} T_{\tau_q} \\
&= \sum_{\tau \in \mathcal{T}_{\tau_q}} A_{\tau^-}^{(\ell-1)} T_\tau,
\end{align*}
which concludes the proof of Equation~\eqref{eq:A_sum_identity}.

Now we prove the walk expansion~\eqref{eq:walk_expansion_lemma} using Equation~\eqref{eq:A_sum_identity}.

\textit{Base case: $\ell = 1$.}
From Equation~\eqref{eq:A_sum_identity}, we have
\begin{equation}
A_t^{(1)} = \sum_{\tau \in \mathcal{T}_{t}} A_{\tau^-}^{(0)} T_\tau = \sum_{\tau \in \mathcal{T}_{t}} T_\tau, \quad \text{for } t > 0,
\end{equation}
which confirms Equation~\eqref{eq:walk_expansion_lemma} for $\ell = 1$.

\textit{Inductive step.}
Assume that Equation~\eqref{eq:walk_expansion_lemma} holds for $\ell - 1$ at all time moments $\tau_1, \ldots, \tau_n$. We show it also holds for $\ell$.

For the case $1 \leq q \leq \ell - 1$ (i.e., $|\mathcal{T}_{\tau_q}| < \ell$), from Equation~\eqref{eq:A_sum_identity} we get
\begin{equation}
A_{\tau_q}^{(\ell)} = \sum_{\tau \in \mathcal{T}_{\tau_q}} A_{\tau^-}^{(\ell-1)} T_\tau = 0,
\end{equation}
since $A_{\tau^-}^{(\ell-1)} = 0$ whenever $|\mathcal{T}_{\tau^-}| < \ell - 1$.

For $\ell \leq q \leq n$, applying Equation~\eqref{eq:A_sum_identity} and the induction hypothesis yields
\begin{equation}
\begin{aligned}
A_{\tau_q}^{(\ell)} 
&= \sum_{\tau \in \mathcal{T}_{\tau_q}} A_{\tau^-}^{(\ell-1)} T_\tau \\
&= A_{\tau_{q-\ell+1}^-}^{(\ell-1)} T_{\tau_{q-\ell+1}} + \cdots + A_{\tau_{q-1}}^{(\ell-1)} T_{\tau_q} \\
&= \sum_{\substack{r_1, \ldots, r_\ell \in \{\tau_1, \ldots, \tau_q\} \\ r_1 < r_2 < \cdots < r_\ell}} T_{r_1} T_{r_2} \cdots T_{r_\ell}.
\end{aligned}
\end{equation}
This concludes the proof of the lemma.
\end{proof}

\begin{proof}
We first note that $A_t^{(\ell)}(i,j)$ denotes a scalar diffusion score, while $\Delta_t^{i}$ is a matrix-valued increment, hence
$\sum_{i \in \mathcal{V}} A_t^{(\ell)}(i,j)\,\Delta_t^{i}\in\mathbb{R}^{d\times d}$
in Equation~\eqref{eq:S_ret} is well-defined as a weighted sum of matrices. We prove the identity in Equation~\eqref{eq:S_closed_ret} by mathematical induction.

\textit{Base case.}
Let $t = \tau_1$, where $\tau_1 > 0$ is the first event timestamp. It follows that $\mathcal{T}_{t} = \{\tau_1\}$ and $\mathcal{T}_{(\tau_1, t]} = \emptyset$.

Using the recursion in~\cref{eq:S_ret}, we have for all $j$ and $\ell$ that
\begin{equation}
\label{eq:S_closed_ret_proof}
S_t^{j,(\ell)} = a_t \cdot 0 + b_t \bar \Delta_t^{j,(\ell)} = \sum_{\tau \in \mathcal{T}_{t}} \left( \prod_{r \in \mathcal{T}_{(\tau, t]}} a_r \right) b_\tau \bar \Delta_{\tau}^{j,(\ell)},
\end{equation}
where the summation contains a single term corresponding to $\tau = \tau_1$, and we adopt the convention that $\prod_{r \in \mathcal{T}_{(\tau_1, t]}} a_r = 1$ whenever $\mathcal{T}_{(\tau_1, t]} = \emptyset$.

\textit{Inductive step.}
Assume that the identity in Equation~\eqref{eq:S_closed_ret} holds for $t^-$ and demonstrate that it also holds for $t$:
\begin{align}
S_t^{j,(\ell)} &= a_t \left[ \sum_{\tau \in \mathcal{T}_{t^-}} \left( \prod_{r \in \mathcal{T}_{(\tau, t^-]}} a_r \right) b_\tau \bar \Delta_{\tau}^{j,(\ell)} \right] + b_t \cdot \bar \Delta_t^{j,(\ell)} \nonumber \\
&= \sum_{\tau \in \mathcal{T}_{t^-}} \left( \prod_{r \in \mathcal{T}_{(\tau, t]}} a_r \right) b_\tau \cdot \bar \Delta_{\tau}^{j,(\ell)} + \sum_{\tau = t} \underbrace{\left( \prod_{r \in \mathcal{T}_{(t, t]}} a_r \right)}_{= 1} b_\tau \cdot \bar \Delta_{\tau}^{j,(\ell)} \nonumber \\
&= \sum_{\tau \in \mathcal{T}_{t}} \left( \prod_{r \in \mathcal{T}_{(\tau, t]}} a_r \right) b_\tau \cdot \bar \Delta_{\tau}^{j,(\ell)},
\end{align}
which concludes the proof of Equation~\eqref{eq:S_closed_ret}.
\end{proof}

\begin{remark}
\label{rem:second_order}
It remains to justify that $A_t^{(\ell)}$ in Equation~\eqref{eq:A_rec_ret} indeed represents time-respecting diffusion scores of increasing walk length. By Lemma~\ref{lem:walk_expansion}, for $\ell, t > 0$ with $|\mathcal{T}_{t}| \geq \ell$:
\begin{equation}
A_t^{(\ell)}
=
\sum_{\substack{\tau_1, \ldots, \tau_\ell \in \mathcal{T}_{t} \\ \tau_1 < \tau_2 < \cdots < \tau_\ell}}
T_{\tau_1} T_{\tau_2} \cdots T_{\tau_\ell},
\label{eq:A_walk_expansion}
\end{equation}
where the constraint $\tau_1 < \tau_2 < \cdots < \tau_\ell$ arises because each recursive step uses $A_{\tau^-}^{(\ell-1)}$, ensuring that the $(\ell-1)$-hop prefix is computed strictly before time $\tau_\ell$.

Since each $T_{\tau}$ is an instantaneous transition matrix induced by events at time $\tau$, the matrix product $T_{\tau_1} \cdots T_{\tau_\ell}$ sums over all length-$\ell$ paths whose step times satisfy $\tau_1 < \cdots < \tau_\ell$. Therefore, each entry $A_t^{(\ell)}(i,j)$ equals the total count (or weighted score, if $T_\tau$ entries are weighted) of all time-respecting walks from $i$ to $j$ of length $\ell$ up to time $t$. Combining Equation~\eqref{eq:S_closed_ret_proof} with Equation~\eqref{eq:A_walk_expansion}, $S_t^{j,(\ell)}$ becomes a temporally retentive accumulation of matrix-valued increments $\Delta_\tau^{i}$ whose influence is propagated to $j$ through length-$\ell$ time-respecting walks, and the readout $O_t^{j}$ performs the local contraction with $Q_t^{j}$.

The increment $\Delta_t^{i}$ aggregates rank-one outer products $(K_{t_e}^{i})^{\!\top} V_{t_e}^{i}$, a standard second-order feature construction used to capture multiplicative channel interactions \cite{lin2015bilinear,gao2016compact,dai2019second,peng2023rwkv,peng2025rwkv7,sun2023retnet,gu2024mamba}. Proposition~\ref{prop:unified_state_rw} shows that once these second-order increments are injected into the graph and propagated via time-respecting walk scores, the resulting state simultaneously admits a closed-form historical aggregation over time and an explicit high-order structural diffusion over increasing walk length. Notably, the recursion Equation~\eqref{eq:A_rec_ret} with unit coefficients ($a_t = b_t = 1$) recovers exact temporal walk counts, while the state dynamics in \eqref{eq:S_ret} with learnable $(a_t, b_t)$ introduce adaptive temporal weighting that balances recent versus historical information.
\end{remark}

\subsection{Proof of \cref{thm:stability}}
\label{app:proof_stability}

We recall that the aggregated increments in Equation~\eqref{eq:retentive_update} are uniformly bounded under standard architectural constraints:
\begin{itemize}
\item \textbf{Bounded node features:} Input features are typically normalized (e.g., layer normalization or batch normalization) before being fed into the model, ensuring $\|x_v\| \le B_x$ for some constant $B_x$.
\item \textbf{Bounded projection matrices:} Weight decay regularization during training implicitly constrains $||W_{K}||, ||W_{V}||\leq L$.
\item \textbf{Bounded neighborhood size:} Real-world temporal graphs exhibit bounded degree distributions. Moreover, neighbor sampling strategies commonly used in scalable GNN training (e.g., sampling at most $d_{\max}$ neighbors per node) explicitly enforce this constraint.
\end{itemize}

\begin{lemma}
\label{lem:bounded_increments}
Assume the following architectural constraints hold:
\begin{enumerate}[(i)]
    \item Node features are bounded: $\|x_v\| \le B_x$ for all $v \in \mathcal{V}$.
    \item Projection matrices are bounded: $\|W_K\|, \|W_V\| \le L$.
    \item The number of interactions involving each node $i$ is bounded: $|\{(i,j,t_e) \in \mathcal{E}_{ t} \}| \le d_{\max}$ .
\end{enumerate}
Then there exists a constant $M > 0$ such that $\|\bar{\Delta}_{t}^{j,(\ell)}\| \le M$ for all $t$, $j$, and $\ell$.
\end{lemma}

\begin{proof}
Recall from Equation~\eqref{eq:Delta_ret} that each increment $\Delta_t^i$ takes the form:
\begin{equation}
\Delta_t^i = \sum_{(i,j,t_e) \in \mathcal{E}_{ t}} \omega(t, (i,j,t_e)) (K_{t_e}^i)^\top V_{t_e}^i.
\end{equation}
Since $\sigma(\cdot) \in (0,1)$ and $\exp(-\lambda^{(\ell)} (\Delta t_e)^{\alpha^{(\ell)}}) \le 1$ (see Equation~\eqref{eq:short-decay}), each weight satisfies $\omega(t, (i,j,t_e)) \le 1$. For the second-order term:
\begin{equation}
\|(K_{t_e}^i)^\top V_{t_e}^i\| \le \|K_{t_e}^i\| \cdot \|V_{t_e}^i\| \le  L^2 B_x^2.
\end{equation}
Since the number of interactions up to time $t$ that involve node $i$  is bounded, we have: $||\Delta_{t}^{i}|| \leq d_{\max}L^2B_x^2$.

It follows that
\begin{equation}
\Big\|\bar  \Delta_t^{i,(\ell)}\Big\| = \Big \| \sum_{i \in \mathcal{V}}  A_t^{( \ell )}\left(i,j\right) \, \Delta_t^i \Big \|\le \sum_{i \in \mathcal{V}}  A_t^{( \ell )}\left(i,j\right) \, \Big \| \Delta_t^i \Big \| \le d_{\max} \cdot d_{\max} \cdot L^2 B_x^2 = d_{\max}^2 L^2 B_x^2 = M.
\end{equation}
\end{proof}

Based on Lemma~\ref{lem:bounded_increments}, we prove the~\cref{thm:stability}:

\begin{proof}

\textbf{Proof of part (i).}
For any $\tau_{k} \in \mathcal{T}_{t}$, we denote $s_k = \|S_{\tau_k}^{j,(\ell)}\|$. Using Equation~\eqref{eq:retentive_update}, we readily obtain the following inequalities:
\begin{equation}
s_k \le \gamma^{(\ell)} s_{k-1} + (1 - \gamma^{(\ell)}) \|\bar{\Delta}_{t_k}^{j}\| \le \gamma^{(\ell)} s_{k-1} + (1 - \gamma^{(\ell)}) M.
\label{eq:norm_recursion}
\end{equation}

We prove by mathematical induction that $s_n \le (1 - (\gamma^{(\ell)})^n) M$ for all $n \ge 1$.

\textit{Base case} ($n = 1$): Using the initialization $S_0^{j,(\ell)} = 0$ (i.e., $s_0 = 0$), we have
\begin{equation}
s_1 \le (\gamma^{(\ell)}) \cdot 0 + (1 - \gamma^{(\ell)}) M = (1 - \gamma^{(\ell)}) M = (1 - \gamma^{(\ell)}) M.
\end{equation}

Assume $s_{n-1} \le (1 - (\gamma^{(\ell)})^{n-1}) M$ holds. Then by Equation~\eqref{eq:norm_recursion}:
\begin{align}
s_n &\le \gamma^{(\ell)} s_{n-1} + (1 - \gamma^{(\ell)}) M \nonumber\\
&\le \gamma^{(\ell)} (1 - (\gamma^{(\ell)})^{n-1}) M + (1 - \gamma^{(\ell)}) M \nonumber\\
&= (\gamma^{(\ell)} - (\gamma^{(\ell)})^n) M + (1 - \gamma^{(\ell)}) M \nonumber\\
&= (1 - (\gamma^{(\ell)})^n) M.
\end{align}
Since $\gamma^{(\ell)} \in (0,1)$, we have $(1 - (\gamma^{(\ell)})^n) < 1$, hence $\|S_t^{j,(\ell)}\| \le M$.

\textbf{Proof of part (ii).} Under the conditions of Theorem~\ref{thm:stability} and Lemma~\ref{lem:bounded_increments}, $\|S_t^{j,(\ell)}\| \le M$. Combined with $\|Q_t^j\| \le L B_x$, we have:
\begin{equation}
\|O_t^{j,(\ell)}\| = \|Q_t^{j} S_t^{j,(\ell)}\| \le \|Q_t^j\| \cdot \|S_t^{j,(\ell)}\| \le L B_x \cdot M.
\end{equation}

This completes the proof.
\end{proof}

\section{Additional Descriptions of Methodology}
\label{appendix:sketch}
\begin{figure}[htbp]
    \centering
    \begin{subfigure}[b]{0.37\textwidth}
        \centering
        \includegraphics[width=\textwidth]{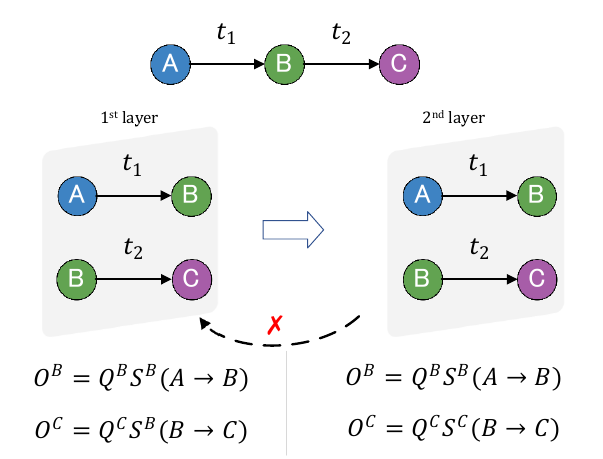}
        \caption{w/o layer-wise propagation} \label{fig:sub_sketch_a}
    \end{subfigure}
    \quad % \hfill 用于在两个子图之间插入弹性空白，将它们撑到两侧
    \begin{subfigure}[b]{0.33\textwidth}
        \centering
        \includegraphics[width=\textwidth]{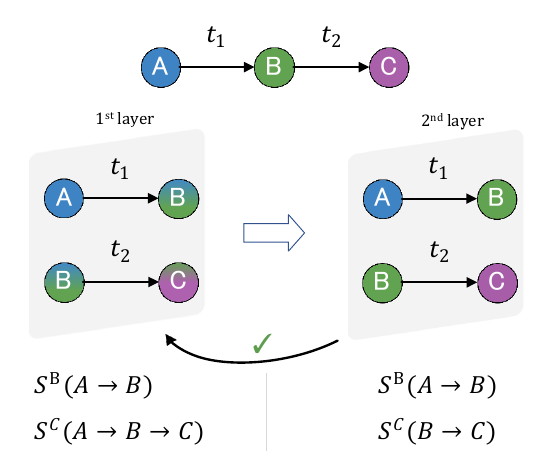}
        \caption{w/ with layer-wise propagation} \label{fig:sub_sketch_b}
    \end{subfigure}
    
\caption{An illustration of topological diffusion via temporal walks.}

    \label{fig:main_sketch}
\end{figure}

\paragraph{Additional Descriptions for Topological Diffusion via Temporal Walks.}

In~\cref{fig:main_sketch}, we provide an illustration of a minimal temporal chain $A \xrightarrow{t_1} B \xrightarrow{t_2} C$ to explain why topological diffusion in dynamic graphs should be realized through temporally valid walks rather than relying on layer-local retention. We contrast two designs that differ solely in whether the retentive state $S$ is propagated across layers.

\textbf{Case (a).}  
In~\cref{fig:sub_sketch_a}, each layer maintains its own retentive state without cross-layer transmission. At the first layer, node $B$ produces its output as $O^B = Q^B S^B(A \rightarrow B)$, while node $C$ produces $O^C = Q^C S^B(B \rightarrow C)$. Here, the superscript of $S$ indicates that both outputs are computed from layer-local states that only encode one-hop interactions at the corresponding time steps. Although the outputs $O^B$ and $O^C$ are forwarded to the next layer, this transmission occurs exclusively through the $Q S$ interaction, which acts as a compressed readout. Consequently, at the second layer, node $C$ can only form $O^C = Q^C S^C(B \rightarrow C)$, where the higher-order dependency $A \rightarrow B \rightarrow C$ has already been discarded. The information originating from $A \rightarrow B$ is never preserved in a persistent state accessible to node $C$, but is instead repeatedly filtered through successive $Q S$ projections, leading to severe structural information loss.

\textbf{Case (b).}  
In~\cref{fig:sub_sketch_b}, the retentive state $S$ is explicitly propagated across layers. Node $B$ first constructs the state $S^B(A \rightarrow B)$, which is then carried forward and extended at the next layer when processing the interaction $B \rightarrow C$. As a result, node $C$ maintains $S^C(A \rightarrow B \rightarrow C)$, which represents a temporally valid walk rather than an isolated one-hop edge. The output is subsequently obtained as $O^C = Q^C S^C(A \rightarrow B \rightarrow C)$, allowing higher-order temporal dependencies to be preserved without repeated compression.

This comparison shows that, in the absence of layer-wise state propagation, multi-layer retentive architectures collapse into stacks of one-hop retention modules whose effective receptive field remains local. Topological diffusion via temporal walks is therefore essential for maintaining long-range structural information in dynamic graphs, which directly motivates the design adopted in this work.

\section{Additional Experimental Descriptions and Results}
\subsection{Datasets}
\label{appendix:datasets}
\begin{table*}[!htb]
\centering
\caption{Statistics of datasets. Type C denotes continuous-time dynamic graphs in which each interaction is associated with a precise timestamp, and Type D denotes discrete-time dynamic graphs whose interactions are aggregated into coarser temporal windows such as days or years. N.\&F. Feat. represents node and link feature dimensions. Graph density is approximated as $|\mathcal{E}|/|\mathcal{V}|^2$ and the scale metric is defined as average links per unique timestep, reflecting the temporal intensity of interactions.}
\label{tab:datasets}
\resizebox{1\columnwidth}{!}{
\begin{tabular}{l|ccccccccccc}
\toprule
Datasets  & Type & Domains & Bipartite & \#Nodes & \#Links & \#N.\&L. Feat. 
& \#Unique Steps & Duration & Time Granularity 
& Density & Scale \\ 
\midrule
Wikipedia  & C & Social & True  & 9227  & 157474  & - \& 172 & 152757 & 1 month & Unix time 
& 1.85e-03 & 1.03 \\

Reddit   & C & Social & True  & 10984 & 672447  & - \& 172 & 669065 & 1 month & Unix time 
& 5.57e-03 & 1.00 \\

MOOC    & C & Interaction & True & 7144  & 411749  & - \& 4  & 345600 & 17 months & Unix time 
& 8.07e-03 & 1.19 \\

LastFM   & C & Interaction & True & 1980  & 1293103 & - \& -  & 1283614 & 1 month & Unix time 
& 3.30e-01 & 1.01 \\

Myket    & C & Interaction & True & 17988 & 694121  & - \& -  & 693774 & 197 days & Unix time 
& 2.15e-03 & 1.00 \\

Enron    & C & Social & False & 184   & 125235  & - \& -  & 22632  & 3 years & Unix time 
& 3.70e+00 & 5.53 \\

Social Evo. & C & Proximity & False & 74 & 2099519 & - \& 2 & 565932 & 8 months & Unix time
& 3.83e+00 & 3.71 \\

UCI     & C & Social & False & 1899  & 59835   & - \& -  & 58911  & 196 days & Unix time 
& 1.66e-02 & 1.01 \\

Flights   & D & Transport & False & 13169 & 1927145 & - \& 1  & 122   & 4 months & days 
& 1.11e-03 & 15799.55 \\

Can. Parl. & D & Politics & False & 734  & 74478  & - \& 1  & 14    & 14 years & years 
& 1.38e-01 & 5320.00 \\

US Legis.  & D & Politics & False & 225  & 60396  & - \& 1  & 12    & 12 congresses & congresses 
& 1.19e+00 & 5033.00 \\

UN Trade  & D & Economics & False & 255  & 507497 & - \& 1  & 32    & 32 years & years 
& 7.79e+00 & 15859.28 \\

UN Vote   & D & Politics & False & 201  & 1035742 & - \& 1 & 72    & 72 years & years 
& 2.56e+01 & 14357.53 \\

Contact   & C & Proximity & False & 692  & 2426279 & - \& 1 & 8064  & 1 month & 5 minutes 
& 5.06e+00 & 300.94 \\
\bottomrule
\end{tabular}}
\end{table*}

We provide below the sources and details of the datasets\footnote{\url{https://zenodo.org/records/7213796\#.Y1cO6y8r30o}} used in this work:

\begin{itemize}

\item \textbf{Wikipedia}: This bipartite CTDG captures a month-long snapshot of editorial activity on Wikipedia, where editors interact with pages through timestamped editing events. Each interaction carries a 172-dimensional LIWC feature vector derived from the edit text, and dynamic labels reflect temporary editing bans imposed on users.

\item \textbf{Reddit}: This dataset chronicles user participation across Reddit communities during a single month. Users and posts form a bipartite CTDG, and each posting event is time-stamped and accompanied by a 172-dimensional LIWC vector. Posting bans appear as dynamic labels on edges.

\item \textbf{MOOC}: Student engagement in an online course ecosystem is represented over 17 months, with interactions linking students to videos, quizzes, or problem sets. Each timestamped access event forms an edge enriched with four interaction features, producing a dense bipartite CTDG.

\item \textbf{LastFM}: User listening behavior over one month is encoded as a CTDG connecting users to songs. Each listening action is treated as a temporal edge. No node or edge features are included, making temporal patterns the primary signal.

\item \textbf{Myket}: This large-scale bipartite CTDG follows app-download activity in the Myket marketplace over 197 days. Users download apps through timestamped interactions, forming edges that carry no additional features beyond timing and identity.

\item \textbf{Enron}: Internal email exchanges within Enron Energy over a three-year period are modeled as a CTDG. Employees serve as nodes, and each email forms a timestamped edge; although featureless, the dataset exhibits rich temporal bursts and communication cycles.

\item \textbf{Social Evo.}: This proximity-based CTDG tracks face-to-face encounters among students in an undergraduate dormitory over eight months. Each interaction is time-stamped and includes two proximity-related features that quantify encounter strength or duration, resulting in a densely connected human-centric temporal network.

\item \textbf{UCI}: This CTDG captures online communication among nearly two thousand UC Irvine students over 196 days. Edges represent time-stamped message exchanges without additional attributes, highlighting temporal messaging dynamics.

\item \textbf{Flights}: Air-traffic patterns during the COVID-19 pandemic are represented as a directed DTDG, with airports as nodes and daily flight volumes forming weighted edges. The network spans four months and exhibits strong temporal variation in global mobility.

\item \textbf{Can. Parl.}: Parliamentary co-voting among Canadian Members of Parliament from 2006 to 2019 produces a political DTDG. Nodes correspond to MPs, and weighted edges track how frequently two MPs agree by voting “yes’’ on the same bill within a year.

\item \textbf{US Legis.}: Legislative collaboration in the US Senate is encoded through co-sponsorship behavior. Senators form the nodes of this DTDG, and edge weights record the number of jointly sponsored bills during each congressional session.

\item \textbf{UN Trade}: This economic DTDG follows agricultural trade flows among 181 countries across 32 years. Directed edges quantify normalized trade values, producing a temporally evolving network dominated by weighted interactions.

\item \textbf{UN Vote}: Voting alignment in the United Nations General Assembly is represented over more than seven decades. Countries are connected whenever they vote “yes” on the same proposal, with edge weights accumulating agreement counts across time.

\item \textbf{Contact}: This proximity-based CTDG monitors physical interactions among roughly 700 university students over one month. Timestamped proximity events form edges whose weights capture the duration or intensity of close-range encounters.

\end{itemize}

\subsection{Baselines}
\label{apppendix:baselines}

The detailed descriptions of the baselines are provided as follows:

\begin{itemize}
\item \textbf{JODIE} learns coupled recurrent updates for users and items and predicts future embeddings through a projection module that models continuous temporal trajectories.

\item \textbf{DyRep} captures graph evolution by modeling communication and association events with temporal point processes, allowing node embeddings to evolve through event-driven message passing.

\item \textbf{TGAT} applies temporal graph attention with a functional time encoding derived from Bochner’s theorem to aggregate temporal–topological information in a continuous-time setting.

\item \textbf{TGN} maintains a memory vector for each node and updates it through temporal message passing, generating time-conditioned embeddings that reflect long-term and short-term dynamics.

\item \textbf{CAWN} represents interactions through causal anonymous walks and encodes temporal walk patterns with recurrent architectures, extracting fine-grained causal and temporal dependencies.

\item \textbf{EdgeBank} stores previously observed edges in memory and predicts future interactions by directly querying the edge history, serving as a simple yet competitive memory-based baseline.

\item \textbf{TCL} introduces a two-stream Transformer that processes the temporal neighborhoods of interacting nodes and applies cross-attention to learn robust temporal representations via contrastive learning.

\item \textbf{GraphMixer} adopts an MLP-Mixer architecture with fixed time encodings and lightweight mixing operations, enabling efficient and stable temporal representation learning without attention or recurrence.

\item \textbf{NAT} incorporates neighborhood-aware temporal aggregation that combines time features with attention-based structural context, producing expressive time-sensitive embeddings.

\item \textbf{PINT} enhances the expressive power of temporal message passing by enforcing injective aggregation functions and encoding relative temporal positions to distinguish complex interaction patterns.

\item \textbf{DyGFormer} models dynamic graphs using a Transformer architecture with aligned temporal encodings and co-occurrence modeling, capturing long-range temporal and structural dependencies.

\item \textbf{TPNet} introduces a time-decay temporal walk matrix and projects it into node representations through random feature propagation, preserving temporal walk statistics efficiently for link prediction.
\end{itemize}

\subsection{Detailed Configurations and Hyperparameters}
\label{appendix:configurations}
For the proposed DSRD, we adopt a unified set of hyperparameters across all experiments, with minor adjustments depending on the dataset scale. DSRD is configured with two model layers and a channel embedding dimension of 64. A dropout rate of 0.1 is used consistently. The number of short-term temporal neighbors is 20. All other hyperparameters follow the global training configuration, including a learning rate of $10^{-4}$, batch size of 200, early stopping with a patience of~10, and the Adam optimizer. For baseline models, we use the configurations and hyperparameters reported in~\citet{yu2023dygformer} and~\citet{lu2024tpnet}, which we verified through independent validation.

\subsection{Additional Results of Transductive Dynamic Link Prediction}
To complement the main findings presented in Section~\ref{section-5:comparison} under the transductive setting, we report the additional ROC-AUC scores across all datasets in Table~\ref{tab:transductive_auc}.

% requires \usepackage[table]{xcolor} in preamble
\begin{table*}[!tb]
\centering
\definecolor{mygreen}{RGB}{122,197,180}
\caption{AUC for transductive dynamic link prediction with random, historical, and inductive negative sampling strategies (NSS). Avg. Rank is computed across datasets for each NSS. The best model is shown in \textbf{bold} and the second best model is \underline{underlined}. \colorbox{mygreen!50}{Shaded cells} indicate relative performance across top 50\% models on a given dataset.}
\label{tab:transductive_auc}
\resizebox{1.01\textwidth}{!}
{
\setlength{\tabcolsep}{0.9mm}
{
\begin{tabular}{c|c|ccccccccccccc}
\toprule
NSS & Datasets & JODIE & DyRep & TGAT & TGN & CAWN & EdgeBank & TCL & GraphMixer & PINT & NAT & DyGFormer & TPNet & DSRD \\
\midrule
\multirow{14}{*}{rnd} & Wikipedia & 93.80{\scriptsize$\pm$0.53} & 92.79{\scriptsize$\pm$0.42} & 90.60{\scriptsize$\pm$0.12} & 95.84{\scriptsize$\pm$0.07} & \cellcolor{mygreen!6}97.74{\scriptsize$\pm$0.01} & 91.50{\scriptsize$\pm$0.00} & 91.14{\scriptsize$\pm$0.52} & 95.85{\scriptsize$\pm$0.08} & \cellcolor{mygreen!13}97.88{\scriptsize$\pm$0.06} & 97.61{\scriptsize$\pm$0.10} & \cellcolor{mygreen!35}98.36{\scriptsize$\pm$0.01} & \underline{\cellcolor{mygreen!70}99.10{\scriptsize$\pm$0.01}} & \textbf{\cellcolor{mygreen!100}99.74{\scriptsize$\pm$0.07}} \\
 & Reddit & 98.18{\scriptsize$\pm$0.07} & 98.15{\scriptsize$\pm$0.04} & 96.04{\scriptsize$\pm$0.29} & 98.64{\scriptsize$\pm$0.02} & 98.48{\scriptsize$\pm$0.04} & 95.99{\scriptsize$\pm$0.00} & 90.89{\scriptsize$\pm$0.25} & 96.70{\scriptsize$\pm$0.01} & \cellcolor{mygreen!35}99.12{\scriptsize$\pm$0.03} & \cellcolor{mygreen!22}98.94{\scriptsize$\pm$0.12} & \cellcolor{mygreen!2}98.67{\scriptsize$\pm$0.03} & \underline{\cellcolor{mygreen!50}99.31{\scriptsize$\pm$0.00}} & \textbf{\cellcolor{mygreen!100}99.98{\scriptsize$\pm$0.00}} \\
 & MOOC & 86.41{\scriptsize$\pm$0.58} & 84.66{\scriptsize$\pm$1.28} & 81.89{\scriptsize$\pm$0.58} & \cellcolor{mygreen!35}90.94{\scriptsize$\pm$0.59} & 76.25{\scriptsize$\pm$0.11} & 62.22{\scriptsize$\pm$0.00} & 80.70{\scriptsize$\pm$0.14} & 82.80{\scriptsize$\pm$0.12} & \cellcolor{mygreen!31}90.32{\scriptsize$\pm$1.56} & \cellcolor{mygreen!11}87.86{\scriptsize$\pm$0.67} & 78.68{\scriptsize$\pm$1.31} & \underline{\cellcolor{mygreen!55}93.51{\scriptsize$\pm$0.79}} & \textbf{\cellcolor{mygreen!100}99.21{\scriptsize$\pm$0.20}} \\
\bottomrule
\end{tabular}
}
}
\end{table*}

\subsection{Experimental Results of Inductive Dynamic Link Prediction}
The results of inductive dynamic link prediction are reported in Tables~\ref{tab:inductive_ap} and~\ref{tab:inductive_auc}.

% requires \usepackage[table]{xcolor} in preamble
\begin{table*}[!ht]
\centering
\definecolor{mygreen}{RGB}{122,197,180}
\caption{AP for inductive dynamic link prediction with random, historical, inductive negative sampling strategies (NSS). Avg. Rank is computed across datasets for each NSS. The best model is shown in \textbf{bold} and the second best model is \underline{underlined}. \colorbox{mygreen!50}{Shaded cells} indicate relative performance across top 50\% models on a given dataset.}
\label{tab:inductive_ap}
\resizebox{1.01\textwidth}{!}
{
\setlength{\tabcolsep}{0.9mm}
{
\begin{tabular}{c|c|cccccccccccc}
\toprule
NSS & Datasets & JODIE & DyRep & TGAT & TGN & CAWN & TCL & GraphMixer & PINT & NAT & DyGFormer & TPNet & DSRD \\
\midrule
\multirow{14}{*}{rnd} & Wikipedia & 89.46{\scriptsize$\pm$0.25} & 87.45{\scriptsize$\pm$0.24} & 90.31{\scriptsize$\pm$0.30} & 93.90{\scriptsize$\pm$0.08} & \cellcolor{mygreen!58}97.70{\scriptsize$\pm$0.01} & 91.16{\scriptsize$\pm$0.32} & 95.10{\scriptsize$\pm$0.05} & \cellcolor{mygreen!60}97.77{\scriptsize$\pm$0.09} & 91.97{\scriptsize$\pm$0.11} & \cellcolor{mygreen!63}97.89{\scriptsize$\pm$0.06} & \underline{\cellcolor{mygreen!81}98.72{\scriptsize$\pm$0.00}} & \textbf{\cellcolor{mygreen!100}99.55{\scriptsize$\pm$0.10}} \\
 & Reddit & 93.21{\scriptsize$\pm$0.23} & 91.76{\scriptsize$\pm$0.25} & 92.66{\scriptsize$\pm$0.50} & 93.48{\scriptsize$\pm$0.27} & \cellcolor{mygreen!64}98.03{\scriptsize$\pm$0.04} & 85.47{\scriptsize$\pm$0.30} & 94.55{\scriptsize$\pm$0.03} & \underline{\cellcolor{mygreen!71}98.36{\scriptsize$\pm$0.05}} & 94.29{\scriptsize$\pm$0.09} & \cellcolor{mygreen!69}98.30{\scriptsize$\pm$0.01} & \cellcolor{mygreen!70}98.33{\scriptsize$\pm$0.07} & \textbf{\cellcolor{mygreen!100}99.94{\scriptsize$\pm$0.02}} \\
 & MOOC & 76.27{\scriptsize$\pm$3.23} & 71.37{\scriptsize$\pm$1.27} & 78.42{\scriptsize$\pm$0.50} & \cellcolor{mygreen!35}85.20{\scriptsize$\pm$0.76} & 73.46{\scriptsize$\pm$0.43} & 77.08{\scriptsize$\pm$0.13} & 79.41{\scriptsize$\pm$0.19} & \cellcolor{mygreen!46}86.88{\scriptsize$\pm$0.59} & \cellcolor{mygreen!6}80.38{\scriptsize$\pm$0.47} & 75.34{\scriptsize$\pm$1.69} & \underline{\cellcolor{mygreen!53}88.05{\scriptsize$\pm$0.69}} & \textbf{\cellcolor{mygreen!100}95.75{\scriptsize$\pm$0.66}} \\
\bottomrule
\end{tabular}
}
}
\end{table*}
% requires \usepackage[table]{xcolor} in preamble
\begin{table*}[!ht]
\centering
\definecolor{mygreen}{RGB}{122,197,180}
\caption{AUC for inductive dynamic link prediction with random, historical, inductive negative sampling strategies (NSS). Avg. Rank is computed across datasets for each NSS. The best model is shown in \textbf{bold} and the second best model is \underline{underlined}. \colorbox{mygreen!50}{Shaded cells} indicate relative performance across top 50\% models on a given dataset.}
\label{tab:inductive_auc}
\resizebox{1.01\textwidth}{!}
{
\setlength{\tabcolsep}{0.9mm}
{
\begin{tabular}{c|c|cccccccccccc}
\toprule
NSS & Datasets & JODIE & DyRep & TGAT & TGN & CAWN & TCL & GraphMixer & PINT & NAT & DyGFormer & TPNet & DSRD \\
\midrule
\multirow{14}{*}{rnd} & Wikipedia & 88.42{\scriptsize$\pm$0.19} & 85.71{\scriptsize$\pm$0.27} & 89.69{\scriptsize$\pm$0.28} & 93.37{\scriptsize$\pm$0.06} & \cellcolor{mygreen!48}97.03{\scriptsize$\pm$0.02} & 90.47{\scriptsize$\pm$0.34} & 94.73{\scriptsize$\pm$0.06} & \cellcolor{mygreen!56}97.41{\scriptsize$\pm$0.13} & 90.11{\scriptsize$\pm$0.22} & \cellcolor{mygreen!59}97.58{\scriptsize$\pm$0.05} & \underline{\cellcolor{mygreen!81}98.63{\scriptsize$\pm$0.02}} & \textbf{\cellcolor{mygreen!100}99.52{\scriptsize$\pm$0.10}} \\
 & Reddit & 92.72{\scriptsize$\pm$0.21} & 91.14{\scriptsize$\pm$0.18} & 92.99{\scriptsize$\pm$0.48} & 92.91{\scriptsize$\pm$0.27} & \cellcolor{mygreen!58}97.58{\scriptsize$\pm$0.07} & 86.38{\scriptsize$\pm$0.27} & 94.37{\scriptsize$\pm$0.02} & \underline{\cellcolor{mygreen!67}98.09{\scriptsize$\pm$0.06}} & 92.72{\scriptsize$\pm$0.03} & \cellcolor{mygreen!65}97.99{\scriptsize$\pm$0.04} & \cellcolor{mygreen!65}97.99{\scriptsize$\pm$0.07} & \textbf{\cellcolor{mygreen!100}99.94{\scriptsize$\pm$0.02}} \\
 & MOOC & 80.40{\scriptsize$\pm$2.43} & 76.37{\scriptsize$\pm$1.08} & 79.57{\scriptsize$\pm$0.63} & \cellcolor{mygreen!48}87.84{\scriptsize$\pm$0.65} & 73.00{\scriptsize$\pm$0.59} & 78.30{\scriptsize$\pm$0.22} & 80.98{\scriptsize$\pm$0.17} & \underline{\cellcolor{mygreen!56}88.96{\scriptsize$\pm$0.59}} & \cellcolor{mygreen!4}81.57{\scriptsize$\pm$0.46} & 74.12{\scriptsize$\pm$2.55} & \cellcolor{mygreen!45}87.37{\scriptsize$\pm$0.97} & \textbf{\cellcolor{mygreen!100}95.28{\scriptsize$\pm$0.74}} \\
\bottomrule
\end{tabular}
}
}
\end{table*}

\subsection{Statistical Significance of Average Ranking Comparisons}
\label{appendix:stats_rank}
Based on the critical difference diagrams in Figures~\ref{fig:CD_AP} and~\ref{fig:CD_AUC}, several observations can be drawn from a modeling perspective. Overall, the proposed DSRD ranks first or within the top tier across both AP and ROC-AUC metrics under all six evaluation settings. The Friedman test with Holm-corrected post-hoc comparisons~\cite{demvsar2006statistical} reveals clear statistically significant gaps between DSRD and the majority of baselines. As visually indicated by the horizontal bars in Figures~\ref{fig:CD_AP} and~\ref{fig:CD_AUC}, DSRD is not connected to most competing methods, confirming that its performance advantage is statistically significant at the 95\% confidence level. Notably, the separation between DSRD and lower-ranked methods is substantial across all settings, with rank differences often exceeding 4--6 positions. Even compared to recent competitive baselines such as TPNet and DyGFormer, DSRD maintains a consistent lead, particularly under challenging inductive evaluation protocols where many baselines exhibit performance degradation. These results indicate that the performance gains of DSRD are systematic rather than driven by isolated datasets or sampling strategies.

Across different negative sampling strategies and evaluation protocols, DSRD exhibits notably smaller rank variance than competing methods. In contrast, many memory-based and attention-based baselines suffer pronounced performance degradation under inductive evaluation or harder negative sampling. This behavior suggests that DSRD is less reliant on memorizing historical edges and instead learns transferable temporal–structural patterns that generalize to unseen nodes and shifted interaction distributions. Comparisons with heuristic and memorization-oriented methods further support this interpretation. While EdgeBank achieves competitive rankings in some transductive settings, its performance deteriorates substantially under inductive sampling, reflecting its dependence on repeated interactions. DSRD, by contrast, maintains strong rankings in inductive scenarios, highlighting the advantage of its retentive state formulation over pure historical lookup.

Among methods designed to capture higher-order structure, such as CAWN, NAT, PINT, DyGformer, and TPNet, competitive performance is observed in certain settings, confirming the importance of multi-hop context in dynamic graphs. However, these approaches exhibit larger fluctuations across evaluation regimes, consistent with their reliance on fixed walk lengths, propagation depths, or static decay schemes. By integrating structural diffusion into a unified retentive state with adaptive temporal and topological regulation, DSRD achieves more stable and consistent performance across heterogeneous graph dynamics.

\begin{figure}[!ht]   % 你也可以改成 [h] 或 [!htbp]
    \centering
    \scalebox{0.8}{
\begin{tikzpicture}[
  group line/.style=semithick
]
\begin{axis}[
  clip=false,
  grid=both,
  axis line style=draw=none,
  tick style=draw=none,
  xticklabel pos=upper,
  y dir=reverse,
  xmin=0.5,
  ymin=0.66,
  legend style={draw=none,fill=none,at={(axis cs:13.0,3.2)},anchor=west,row sep=.35em,legend columns=1},
  legend cell align=left,
  title style={yshift=\baselineskip},
  width=0.7\linewidth,
  ytick={1,2,3,4,5,6},
  yticklabels={{Random NSS (Transductive)},{Historical NSS (Transductive)},{Inductive NSS (Transductive)},{Random NSS (Inductive)},{Historical NSS (Inductive)},{Inductive NSS (Inductive)}},
  xmax=12.5,
  ymax=6.66,
  height=1.0*\axisdefaultheight,
  cycle list={{color1,mark=+},{color2,mark=diamond*},{color3,mark=triangle*,semithick},{color4,mark=square*,semithick},{color5,mark=triangle*,semithick},{color6,mark=star},{color7,mark=otimes*},{color8,mark=x},{color9,mark=Mercedes star},{color10,mark=o,only marks},{color11,mark=halfcircle*},{color12,mark=pentagon*,semithick},{color13,mark=pentagon*}},
  title=AP
]
\addplot+[only marks] coordinates {
  (7.285714285714286, 1)
  (8.0, 2)
  (8.5, 3)
  (8.142857142857142, 4)
  (8.142857142857142, 5)
  (8.285714285714286, 6)
};
\addlegendentry{JODIE}
\addplot+[only marks] coordinates {
  (9.5, 1)
  (7.928571428571429, 2)
  (8.5, 3)
  (9.928571428571429, 4)
  (8.714285714285714, 5)
  (8.714285714285714, 6)
};
\addlegendentry{DyRep}
\addplot+[only marks] coordinates {
  (10.0, 1)
  (8.5, 2)
  (7.071428571428571, 3)
  (9.0, 4)
  (6.5, 5)
  (6.5, 6)
};
\addlegendentry{TGAT}
\addplot+[only marks] coordinates {
  (7.071428571428571, 1)
  (5.928571428571429, 2)
  (6.357142857142857, 3)
  (8.785714285714286, 4)
  (7.5, 5)
  (7.5, 6)
};
\addlegendentry{TGN}
\addplot+[only marks] coordinates {
  (8.0, 1)
  (10.428571428571429, 2)
  (8.357142857142858, 3)
  (7.0, 4)
  (8.857142857142858, 5)
  (8.857142857142858, 6)
};
\addlegendentry{CAWN}
\addplot+[only marks] coordinates {
  (9.928571428571429, 1)
  (8.857142857142858, 2)
  (8.214285714285714, 3)
  (8.785714285714286, 4)
  (6.785714285714286, 5)
  (6.642857142857143, 6)
};
\addlegendentry{TCL}
\addplot+[only marks] coordinates {
  (7.857142857142857, 1)
  (6.928571428571429, 2)
  (6.571428571428571, 3)
  (7.285714285714286, 4)
  (5.142857142857143, 5)
  (5.142857142857143, 6)
};
\addlegendentry{GraphMixer}
\addplot+[only marks] coordinates {
  (3.9285714285714284, 1)
  (4.285714285714286, 2)
  (5.357142857142857, 3)
  (4.5, 4)
  (5.785714285714286, 5)
  (5.785714285714286, 6)
};
\addlegendentry{PINT}
\addplot+[only marks] coordinates {
  (5.285714285714286, 1)
  (6.0, 2)
  (7.5, 3)
  (5.857142857142857, 4)
  (8.642857142857142, 5)
  (8.642857142857142, 6)
};
\addlegendentry{NAT}
\addplot+[only marks] coordinates {
  (6.035714285714286, 1)
  (6.642857142857143, 2)
  (6.857142857142857, 3)
  (4.571428571428571, 4)
  (6.357142857142857, 5)
  (6.357142857142857, 6)
};
\addlegendentry{DyGFormer}
\addplot+[only marks] coordinates {
  (1.9642857142857142, 1)
  (3.357142857142857, 2)
  (3.357142857142857, 3)
  (2.4285714285714284, 4)
  (4.142857142857143, 5)
  (4.142857142857143, 6)
};
\addlegendentry{TPNet}
\addplot+[only marks] coordinates {
  (1.1428571428571428, 1)
  (1.1428571428571428, 2)
  (1.3571428571428572, 3)
  (1.7142857142857142, 4)
  (1.4285714285714286, 5)
  (1.4285714285714286, 6)
};
\addlegendentry{DSRD}
\draw[group line] (axis cs:6.035714285714286,1.1044303797468356) -- ++(0pt,-3pt) -- ([yshift=-3pt]axis cs:10.0,1.1044303797468356) -- ++(0pt,3pt);
\draw[group line] (axis cs:1.1428571428571428,1.2626582278481013) -- ++(0pt,-3pt) -- ([yshift=-3pt]axis cs:1.9642857142857142,1.2626582278481013) -- ++(0pt,3pt);
\draw[group line] (axis cs:5.285714285714286,1.4208860759493671) -- ++(0pt,-3pt) -- ([yshift=-3pt]axis cs:8.0,1.4208860759493671) -- ++(0pt,3pt);
\draw[group line] (axis cs:3.9285714285714284,1.5791139240506329) -- ++(0pt,-3pt) -- ([yshift=-3pt]axis cs:7.857142857142857,1.5791139240506329) -- ++(0pt,3pt);
\draw[group line] (axis cs:4.285714285714286,2.180327868852459) -- ++(0pt,-3pt) -- ([yshift=-3pt]axis cs:10.428571428571429,2.180327868852459) -- ++(0pt,3pt);
\draw[group line] (axis cs:3.357142857142857,2.4535519125683063) -- ++(0pt,-3pt) -- ([yshift=-3pt]axis cs:8.0,2.4535519125683063) -- ++(0pt,3pt);
\draw[group line] (axis cs:3.357142857142857,3.2832618025751072) -- ++(0pt,-3pt) -- ([yshift=-3pt]axis cs:8.5,3.2832618025751072) -- ++(0pt,3pt);
\draw[group line] (axis cs:4.5,4.1044303797468356) -- ++(0pt,-3pt) -- ([yshift=-3pt]axis cs:9.0,4.1044303797468356) -- ++(0pt,3pt);
\draw[group line] (axis cs:1.7142857142857142,4.262658227848101) -- ++(0pt,-3pt) -- ([yshift=-3pt]axis cs:4.571428571428571,4.262658227848101) -- ++(0pt,3pt);
\draw[group line] (axis cs:2.4285714285714284,4.420886075949367) -- ++(0pt,-3pt) -- ([yshift=-3pt]axis cs:5.857142857142857,4.420886075949367) -- ++(0pt,3pt);
\draw[group line] (axis cs:4.571428571428571,4.579113924050633) -- ++(0pt,-3pt) -- ([yshift=-3pt]axis cs:9.928571428571429,4.579113924050633) -- ++(0pt,3pt);
\draw[group line] (axis cs:4.142857142857143,5.283261802575107) -- ++(0pt,-3pt) -- ([yshift=-3pt]axis cs:8.857142857142858,5.283261802575107) -- ++(0pt,3pt);
\draw[group line] (axis cs:4.142857142857143,6.283261802575107) -- ++(0pt,-3pt) -- ([yshift=-3pt]axis cs:8.857142857142858,6.283261802575107) -- ++(0pt,3pt);
\end{axis}
\end{tikzpicture}
}

    \caption{Average AP rank comparison of dynamic graph models across three negative sampling strategies. Statistical significance is assessed using the Friedman test with Holm-corrected post-hoc comparisons at the 95\% confidence level.}
    \label{fig:CD_AP}
\end{figure}
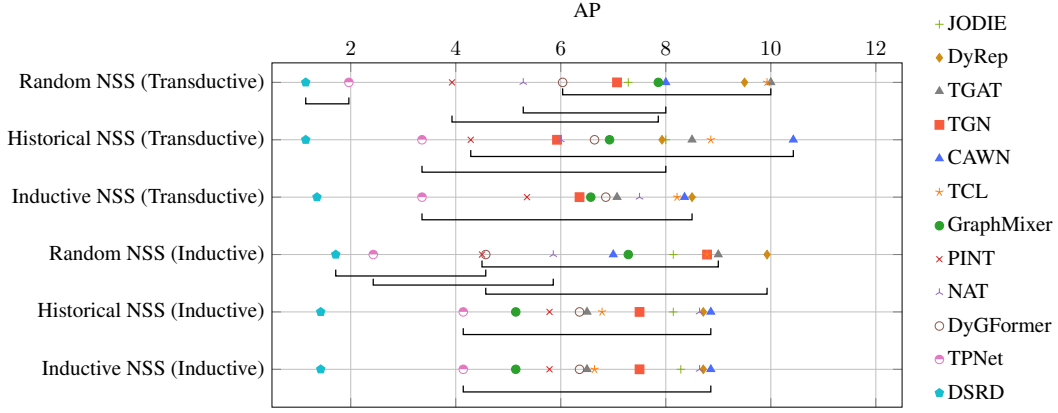

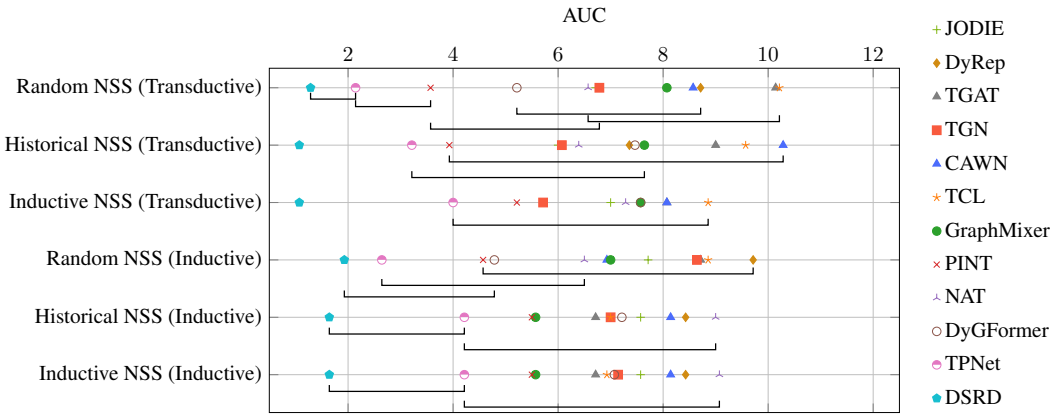
\begin{figure}[!ht]   % 你也可以改成 [h] 或 [!htbp]
    \centering
    \scalebox{0.8}{%
\begin{tikzpicture}[
  group line/.style=semithick
]
\begin{axis}[
  clip=false,
  grid=both,
  axis line style=draw=none,
  tick style=draw=none,
  xticklabel pos=upper,
  y dir=reverse,
  xmin=0.5,
  ymin=0.66,
  legend style={draw=none,fill=none,at={(axis cs:13.0,3.2)},anchor=west,row sep=.35em,legend columns=1},
  legend cell align=left,
  title style={yshift=\baselineskip},
  width=0.7\linewidth,
  ytick={1,2,3,4,5,6},
  yticklabels={{Random NSS (Transductive)},{Historical NSS (Transductive)},{Inductive NSS (Transductive)},{Random NSS (Inductive)},{Historical NSS (Inductive)},{Inductive NSS (Inductive)}},
  xmax=12.5,
  ymax=6.66,
  height=1.0*\axisdefaultheight,
  cycle list={{color1,mark=+},{color2,mark=diamond*},{color3,mark=triangle*,semithick},{color4,mark=square*,semithick},{color5,mark=triangle*,semithick},{color6,mark=star},{color7,mark=otimes*},{color8,mark=x},{color9,mark=Mercedes star},{color10,mark=o,only marks},{color11,mark=halfcircle*},{color12,mark=pentagon*,semithick},{color13,mark=pentagon*}},
  title=AUC
]
\addplot+[only marks] coordinates {
  (6.714285714285714, 1)
  (6.0, 2)
  (7.0, 3)
  (7.714285714285714, 4)
  (7.571428571428571, 5)
  (7.571428571428571, 6)
};
\addlegendentry{JODIE}
\addplot+[only marks] coordinates {
  (8.714285714285714, 1)
  (7.357142857142857, 2)
  (7.571428571428571, 3)
  (9.714285714285714, 4)
  (8.428571428571429, 5)
  (8.428571428571429, 6)
};
\addlegendentry{DyRep}
\addplot+[only marks] coordinates {
  (10.142857142857142, 1)
  (9.0, 2)
  (8.071428571428571, 3)
  (8.714285714285714, 4)
  (6.714285714285714, 5)
  (6.714285714285714, 6)
};
\addlegendentry{TGAT}
\addplot+[only marks] coordinates {
  (6.785714285714286, 1)
  (6.071428571428571, 2)
  (5.714285714285714, 3)
  (8.642857142857142, 4)
  (7.0, 5)
  (7.142857142857143, 6)
};
\addlegendentry{TGN}
\addplot+[only marks] coordinates {
  (8.571428571428571, 1)
  (10.285714285714286, 2)
  (8.071428571428571, 3)
  (6.928571428571429, 4)
  (8.142857142857142, 5)
  (8.142857142857142, 6)
};
\addlegendentry{CAWN}
\addplot+[only marks] coordinates {
  (10.214285714285714, 1)
  (9.571428571428571, 2)
  (8.857142857142858, 3)
  (8.857142857142858, 4)
  (7.0, 5)
  (6.928571428571429, 6)
};
\addlegendentry{TCL}
\addplot+[only marks] coordinates {
  (8.071428571428571, 1)
  (7.642857142857143, 2)
  (7.571428571428571, 3)
  (7.0, 4)
  (5.571428571428571, 5)
  (5.571428571428571, 6)
};
\addlegendentry{GraphMixer}
\addplot+[only marks] coordinates {
  (3.5714285714285716, 1)
  (3.9285714285714284, 2)
  (5.214285714285714, 3)
  (4.571428571428571, 4)
  (5.5, 5)
  (5.5, 6)
};
\addlegendentry{PINT}
\addplot+[only marks] coordinates {
  (6.571428571428571, 1)
  (6.392857142857143, 2)
  (7.285714285714286, 3)
  (6.5, 4)
  (9.0, 5)
  (9.071428571428571, 6)
};
\addlegendentry{NAT}
\addplot+[only marks] coordinates {
  (5.214285714285714, 1)
  (7.464285714285714, 2)
  (7.571428571428571, 3)
  (4.785714285714286, 4)
  (7.214285714285714, 5)
  (7.071428571428571, 6)
};
\addlegendentry{DyGFormer}
\addplot+[only marks] coordinates {
  (2.142857142857143, 1)
  (3.2142857142857144, 2)
  (4.0, 3)
  (2.642857142857143, 4)
  (4.214285714285714, 5)
  (4.214285714285714, 6)
};
\addlegendentry{TPNet}
\addplot+[only marks] coordinates {
  (1.2857142857142858, 1)
  (1.0714285714285714, 2)
  (1.0714285714285714, 3)
  (1.9285714285714286, 4)
  (1.6428571428571428, 5)
  (1.6428571428571428, 6)
};
\addlegendentry{DSRD}
\draw[group line] (axis cs:1.2857142857142858,1.0862745098039215) -- ++(0pt,-3pt) -- ([yshift=-3pt]axis cs:2.142857142857143,1.0862745098039215) -- ++(0pt,3pt);
\draw[group line] (axis cs:2.142857142857143,1.2169934640522877) -- ++(0pt,-3pt) -- ([yshift=-3pt]axis cs:3.5714285714285716,1.2169934640522877) -- ++(0pt,3pt);
\draw[group line] (axis cs:5.214285714285714,1.3477124183006537) -- ++(0pt,-3pt) -- ([yshift=-3pt]axis cs:8.714285714285714,1.3477124183006537) -- ++(0pt,3pt);
\draw[group line] (axis cs:6.571428571428571,1.4784313725490197) -- ++(0pt,-3pt) -- ([yshift=-3pt]axis cs:10.214285714285714,1.4784313725490197) -- ++(0pt,3pt);
\draw[group line] (axis cs:3.5714285714285716,1.6091503267973857) -- ++(0pt,-3pt) -- ([yshift=-3pt]axis cs:6.785714285714286,1.6091503267973857) -- ++(0pt,3pt);
\draw[group line] (axis cs:3.9285714285714284,2.180327868852459) -- ++(0pt,-3pt) -- ([yshift=-3pt]axis cs:10.285714285714286,2.180327868852459) -- ++(0pt,3pt);
\draw[group line] (axis cs:3.2142857142857144,2.4535519125683063) -- ++(0pt,-3pt) -- ([yshift=-3pt]axis cs:7.642857142857143,2.4535519125683063) -- ++(0pt,3pt);
\draw[group line] (axis cs:4.0,3.2832618025751072) -- ++(0pt,-3pt) -- ([yshift=-3pt]axis cs:8.857142857142858,3.2832618025751072) -- ++(0pt,3pt);
\draw[group line] (axis cs:4.571428571428571,4.132264529058116) -- ++(0pt,-3pt) -- ([yshift=-3pt]axis cs:9.714285714285714,4.132264529058116) -- ++(0pt,3pt);
\draw[group line] (axis cs:2.642857142857143,4.332665330661323) -- ++(0pt,-3pt) -- ([yshift=-3pt]axis cs:6.5,4.332665330661323) -- ++(0pt,3pt);
\draw[group line] (axis cs:1.9285714285714286,4.533066132264529) -- ++(0pt,-3pt) -- ([yshift=-3pt]axis cs:4.785714285714286,4.533066132264529) -- ++(0pt,3pt);
\draw[group line] (axis cs:1.6428571428571428,5.180327868852459) -- ++(0pt,-3pt) -- ([yshift=-3pt]axis cs:4.214285714285714,5.180327868852459) -- ++(0pt,3pt);
\draw[group line] (axis cs:4.214285714285714,5.453551912568306) -- ++(0pt,-3pt) -- ([yshift=-3pt]axis cs:9.0,5.453551912568306) -- ++(0pt,3pt);
\draw[group line] (axis cs:1.6428571428571428,6.180327868852459) -- ++(0pt,-3pt) -- ([yshift=-3pt]axis cs:4.214285714285714,6.180327868852459) -- ++(0pt,3pt);
\draw[group line] (axis cs:4.214285714285714,6.453551912568306) -- ++(0pt,-3pt) -- ([yshift=-3pt]axis cs:9.071428571428571,6.453551912568306) -- ++(0pt,3pt);
\end{axis}
\end{tikzpicture}
}
    \caption{Average ROC-AUC rank comparison of dynamic graph models across three negative sampling strategies. All the graphical conventions are the same as in~\cref{fig:CD_AP}}
    \label{fig:CD_AUC}
\end{figure}

\subsection{Additional Description of Dynamic Node Classification}
\label{appendix:node_calssificaiton}
Since node classification datasets exhibit highly imbalanced label distributions, ROC-AUC alone may not fully reflect model performance on minority classes. 
We therefore additionally report the area under the precision--recall curve (AUPRC), which is more sensitive to class imbalance. 
For each dataset, we also report the relative improvement over a random classifier, where the random baseline corresponds to the positive-class ratio.

\begin{table*}[!htb]
\centering
\caption{AUPRC (\%) for node classification under imbalanced label distributions. The value in parentheses denotes the relative improvement over the random baseline determined by the positive-class ratio.}
\label{tab:node_cls_auprc}
\resizebox{0.8\linewidth}{!}{
\begin{tabular}{l|c|ccccc}
\toprule
Dataset & Imbalance & TGN & PINT & DyGFormer & TPNet & DSRD (ours) \\
\midrule
Wikipedia & Severe 
& 1.13 (6.07$\times$) 
& 2.11 (11.33$\times$) 
& 1.42 (7.62$\times$) 
& 0.13 (0.66$\times$) 
& \textbf{2.47 (13.09$\times$)} \\

Reddit & Severe 
& \textbf{0.17 (1.79$\times$)} 
& 0.10 (1.11$\times$) 
& 0.16 (1.64$\times$) 
& 0.11 (1.10$\times$) 
& 0.13 (1.35$\times$) \\

MOOC & Severe 
& 1.75 (1.87$\times$) 
& 1.82 (1.94$\times$) 
& 3.77 (4.02$\times$) 
& 1.17 (1.24$\times$) 
& \textbf{4.10 (4.37$\times$)} \\

PubMed & Moderate 
& 81.88 (1.04$\times$) 
& 82.59 (1.04$\times$) 
& 83.33 (1.05$\times$) 
& 78.11 (0.99$\times$) 
& \textbf{87.48 (1.11$\times$)} \\
\bottomrule
\end{tabular}}
\end{table*}

Table~\ref{tab:node_cls_auprc} shows that DSRD remains competitive under different degrees of class imbalance. 
On MOOC and PubMed, DSRD achieves the best AUPRC, reaching 4.10\% on MOOC and 87.48\% on PubMed, corresponding to 4.37$\times$ and 1.11$\times$ improvement over the random baselines, respectively. 
On Wikipedia, despite the severe imbalance, DSRD also obtains the best AUPRC with 2.47\%, achieving a 13.09$\times$ improvement over random prediction. 
These results indicate that the proposed retentive diffusion mechanism can preserve useful minority-class signals even when positive labels are sparse.

For Reddit, all methods obtain very low AUPRC values, and DSRD is slightly behind the best-performing baseline. 
This dataset represents an extreme imbalance regime, where the positive-class ratio is extremely small and precision--recall performance becomes highly sensitive to a small number of positive predictions. 
The uniformly low AUPRC values across all methods suggest that such severe imbalance remains challenging for current dynamic graph learning models. 
Overall, the AUPRC results complement the ROC-AUC evaluation by showing that DSRD is robust across moderately and severely imbalanced settings, while also revealing the limitations of existing methods under extreme label sparsity.

\subsection{Interpreting High Performance in Temporal Link Prediction}
\label{appendix:saturated_performance}

In this section, we analyze the near-saturated AP and ROC-AUC scores observed for a subset of benchmark and evaluation settings in~\cref{tab:transductive_ap,tab:transductive_auc}. Rather than indicating trivial memorization, these instances of saturation reflect structural properties of commonly used temporal graph datasets and evaluation protocols. We further show that average performance alone may obscure meaningful variations in generalization behavior when evaluation conditions change across temporal, structural, and task-specific regimes.

Many widely adopted temporal graph benchmarks exhibit strong temporal locality and interaction repetition. In particular, a substantial fraction of test edges correspond to node pairs that have appeared previously in the training history, often within relatively short time intervals. Under such conditions, link prediction is mainly influenced by short-term temporal cues, which can lead to near-saturated performance for multiple methods. Notably, such saturation arises only under specific benchmark configurations and does not persist uniformly across all evaluation settings (see~\cref{tab:inductive_ap,tab:inductive_auc,tab:node_class_test_auc}). In addition, the degree of performance saturation is sensitive to the choice of negative sampling strategy. When negatives are sampled in a manner that is structurally or temporally distant from positive edges, the resulting classification boundary becomes relatively coarse, reducing the discriminative power of aggregate metrics such as AP and ROC-AUC. As the negative sampling protocol changes, performance exhibits systematic and interpretable variations, indicating that saturation is conditional on the evaluation setup rather than inherent to the prediction task itself. Further evidence for this interpretation is provided by node-level prediction tasks, where near-saturated performance is not observed and clear performance differences remain across methods. Unlike link prediction, node-level tasks depend more heavily on aggregated temporal and structural context, making them less susceptible to short-term interaction repetition. This contrast reinforces the view that near-saturated link prediction performance is regime-dependent and does not reflect trivial memorization.

\subsection{Detailed Analysis of Adaptive Decay Behaviors}
\label{appendix:decays}

\begin{figure}[!h]
    \centering
    \begin{minipage}[t]{0.48\textwidth}
        \centering
        \includegraphics[width=\linewidth]{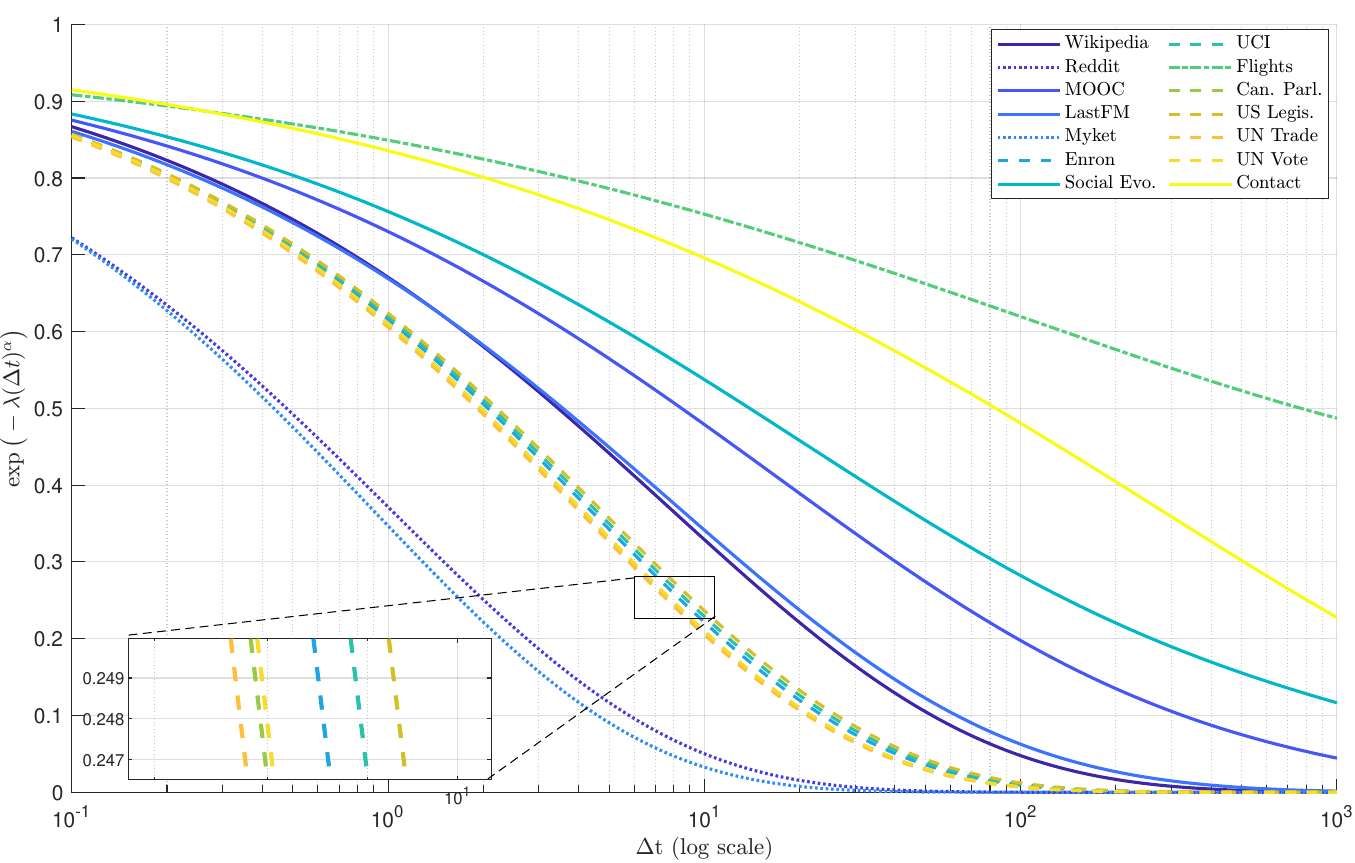}
        \caption{Visualization of learned temporal decay weights $\exp(-\lambda(\Delta t)^\alpha)$ as a function of time interval $\Delta t$ for all datasets at the first layer ($\ell=1$).}
        \label{fig:short}
    \end{minipage}
    \hfill
    % 右边的图
    \begin{minipage}[t]{0.48\textwidth}
        \centering
        \includegraphics[width=\linewidth]{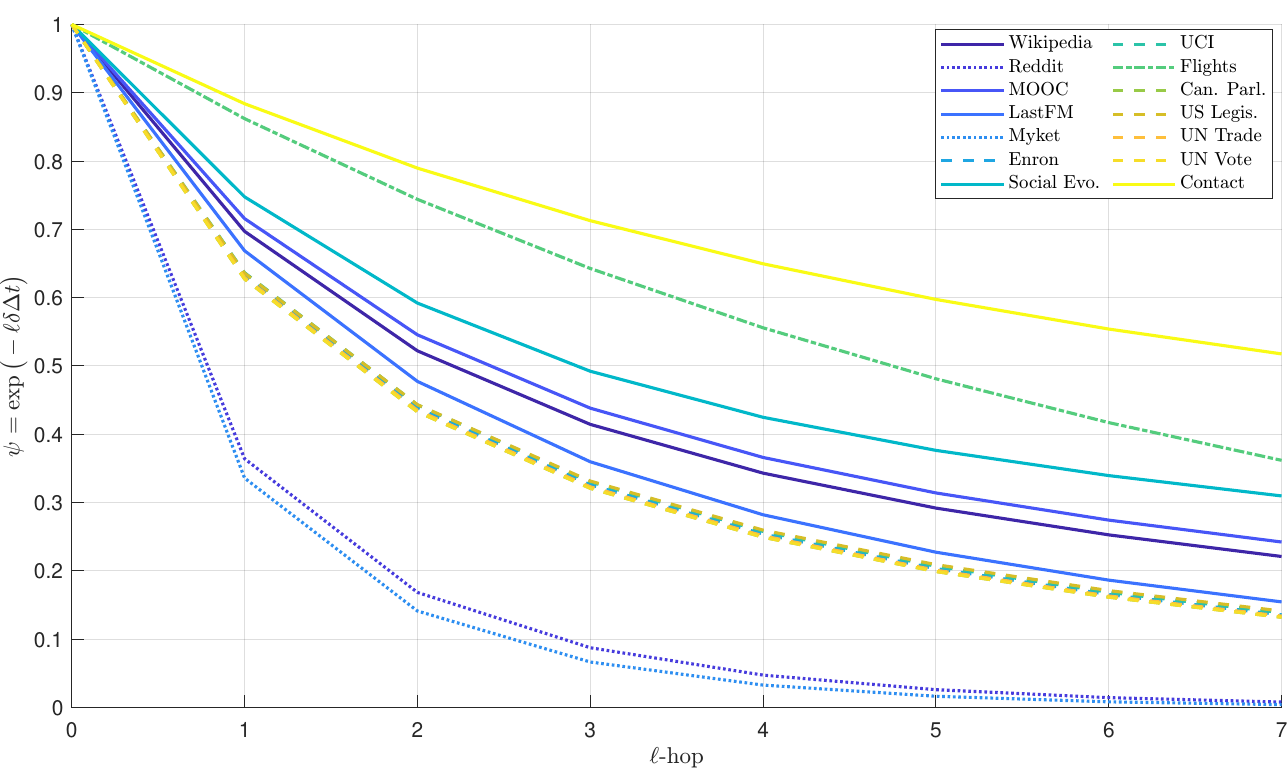}
        \caption{Visualization of learned topological decay weights $\psi = \exp(-\ell \delta^{(\ell)} \Delta t)$ across propagation depths $\ell$ for all datasets, evaluated at $\Delta t = 100$.}
        \label{fig:topo}
    \end{minipage}
\end{figure}

This section provides a comprehensive analysis of the learned decay parameters across all 14 datasets, complementing the summary in Section~\ref{sec:adaptive-decay}. By examining the correlation between learned parameters and dataset properties in Table~\ref{tab:datasets}, we reveal how DSRD automatically adapts to diverse dynamic regimes.

\paragraph{Long-term Retention Decay.}
Figure~\ref{fig:long} shows the learned long-term retention decay $\gamma^{(\ell)}$ across state update steps. DTDG datasets with coarse temporal granularity (years) and high scale ($>$5000), such as UN Trade and US Legis., exhibit strong long-term retention, preserving information across sparse snapshots. Interestingly, some CTDG datasets like Wikipedia also learn strong retention despite fine-grained timestamps, likely reflecting persistent editorial dynamics in social platforms. In contrast, proximity-based graphs with high density ($>$3.0), such as Contact and Social Evo., learn rapid state turnover, capturing the transient nature of face-to-face interactions. Notably, Flights exhibits fast decay despite being a high-scale DTDG, suggesting that transport-domain connection patterns reset across temporal windows rather than accumulating historical dependencies.

\paragraph{Short-term Temporal Decay.}
Figure~\ref{fig:short} shows the learned short-term decay weights $\exp(-\lambda(\Delta t)^\alpha)$ across time intervals. High-density proximity graphs such as Contact and Social Evo., along with the coarse-grained Flights dataset, maintain broad temporal sensitivity, benefiting from smoothing over burst-like interaction patterns or extended time windows. Sparse bipartite graphs with scale $\approx$1.0, such as Reddit and Myket, exhibit the sharpest decay, where individual recent events carry strong predictive signals and older interactions quickly become uninformative. The remaining datasets occupy an intermediate range, reflecting balanced temporal sensitivity across diverse interaction regimes.

\paragraph{Topological Decay.}
Figure~\ref{fig:topo} presents the learned structural decay weights $\psi = \exp(-\ell\delta^{(\ell)} \Delta t)$ across propagation hops. Dense non-bipartite graphs such as Contact and Social Evo.\ (density $>$3.0) maintain broad multi-hop propagation, leveraging rich higher-order connectivity. Flights, despite low overall density, achieves similar broad propagation due to its extremely high scale (15799.55) creating dense within-window structure. Conversely, sparse bipartite graphs including Reddit and Myket (density $<$0.01) learn sharp topological cutoffs, as higher-order paths are sparse and less informative. Political and economic DTDGs, along with most CTDG datasets, show moderate structural receptive fields aligned with their intermediate density characteristics.

\paragraph{Cross-dimension Correlation.}
Comparing the above decay dimensions reveals interpretable joint patterns. Sparse bipartite graphs like Reddit and Myket consistently favor locality across all dimensions, while proximity graphs such as Contact and Social Evo.\ share broad temporal-structural receptive fields but rapid state turnover, suggesting event-driven dynamics where interactions propagate widely but history resets quickly. However, the mapping from graph properties to decay regimes is not fully deterministic: Wikipedia learns atypically strong retention for a CTDG, while Flights behaves unlike other high-scale DTDGs. These exceptions highlight that domain-specific interaction semantics shape the optimal decay regime beyond structural statistics alone.

\begin{figure}[!h]
    \centering
    
    \begin{subfigure}[b]{0.30\textwidth}
        \centering
        \includegraphics[width=\textwidth]{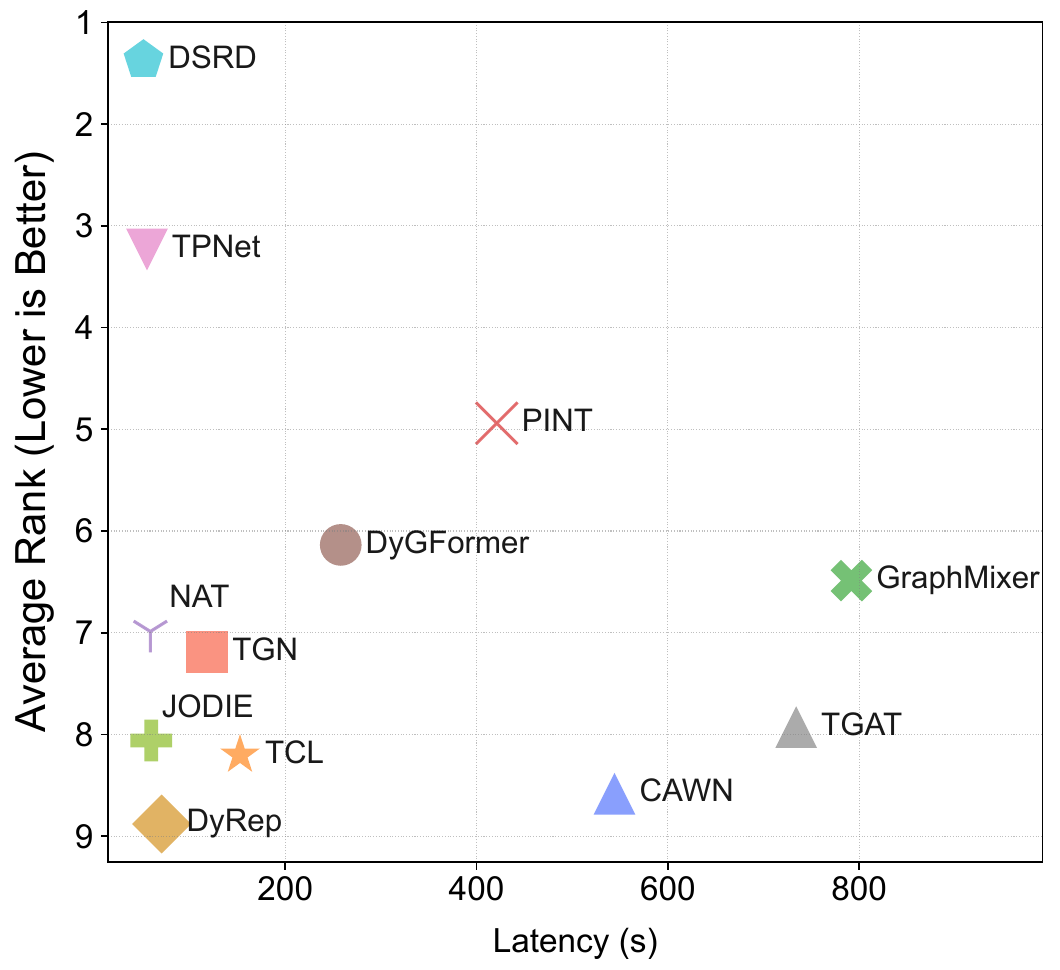}
        \caption{Latency per epoch (s)}
        \label{fig:latency}
    \end{subfigure}
    \hfill
    \begin{subfigure}[b]{0.30\textwidth}
        \centering
        \includegraphics[width=\textwidth]{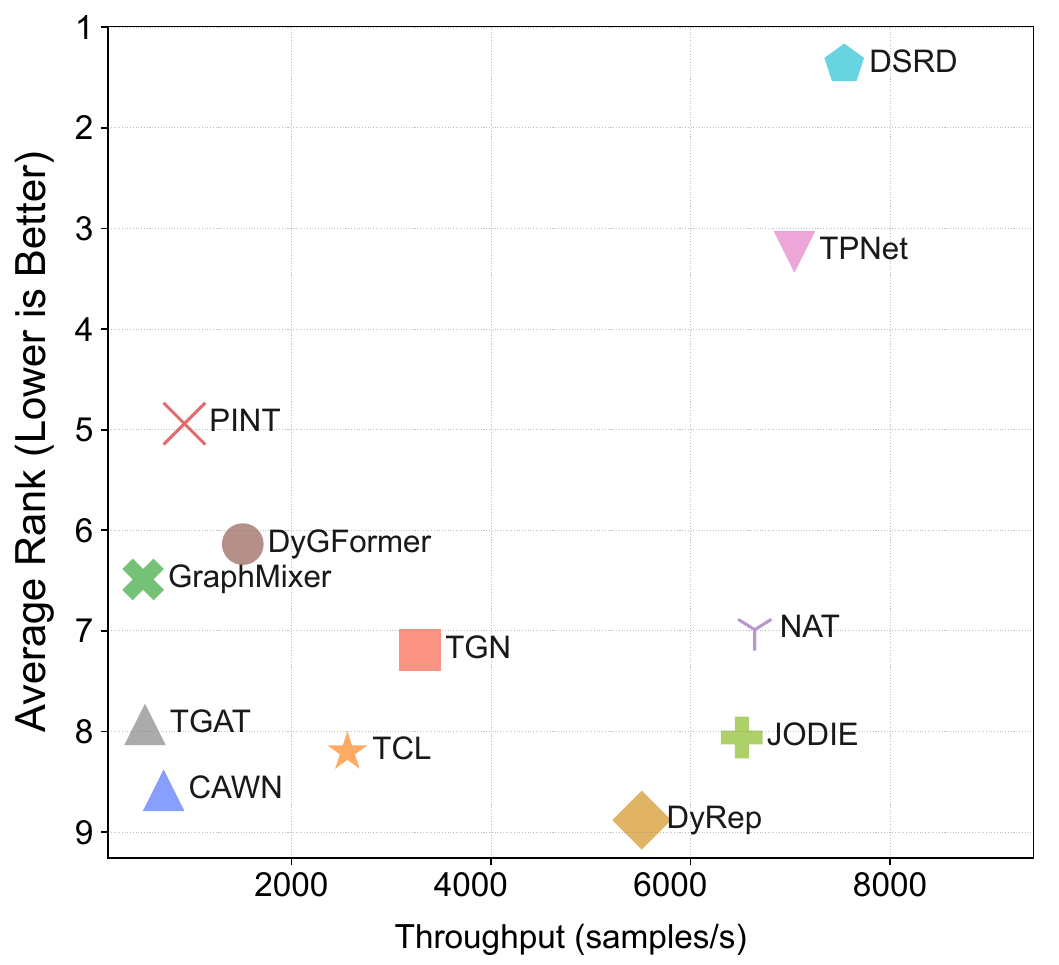}
        \caption{Throughput samples per second}
        \label{fig:throughput}
    \end{subfigure}
    \hfill
    \begin{subfigure}[b]{0.34\textwidth} % 下面这张图通常可以宽一点
        \centering
        \includegraphics[width=\textwidth]{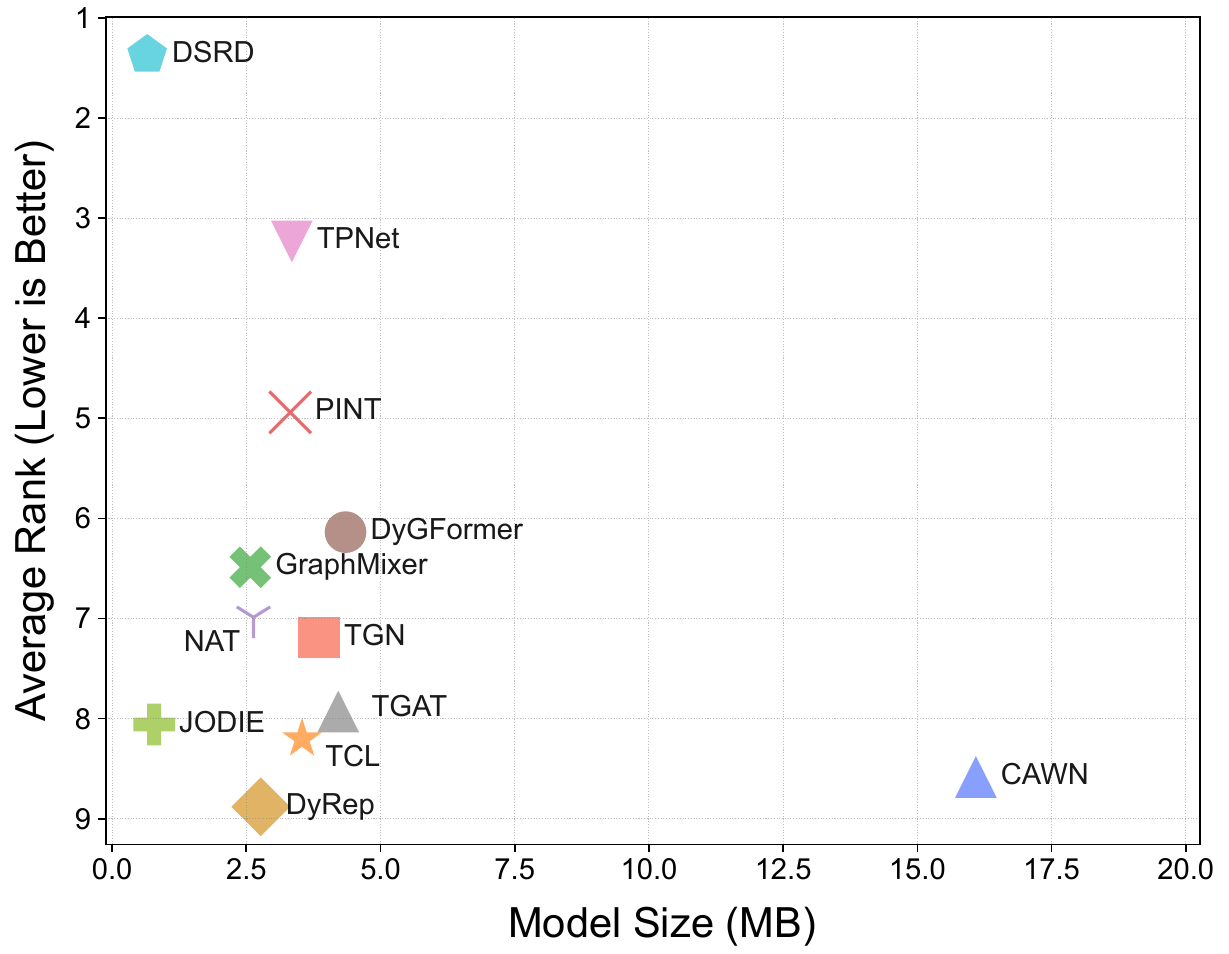}
        \caption{Model Size (MB)}
        \label{fig:model_size}
    \end{subfigure}
    
    \caption{Scalability comparison of dynamic graph methods. (a) Latency, (b) Throughput, and (c) Model size, each plotted against average rank (from Figure~\ref{fig:CD_AP}) across all datasets. The throughput sample in (b) refers to a positive-negative sample pair.}
    \label{fig:scalability}
\end{figure}

\subsection{Scalability}
\label{appendix:scalability}
We evaluate the scalability of dynamic graph learning methods across three dimensions: latency per epoch in Figure~\ref{fig:latency}, throughput (samples processed per second) in Figure~\ref{fig:throughput}, and model size in Figure~\ref{fig:model_size}, each plotted against the average performance rank across datasets. As shown in Figure~\ref{fig:scalability}, DSRD achieves a favorable balance between efficiency and accuracy. It ranks among the best-performing models while maintaining moderate latency and high throughput, significantly outperforming heavy architectures such as PINT and TPNet in runtime. Moreover, the compact model size of DSRD ensures lightweight deployment without compromising predictive power. These results confirm the practicality of DSRD in large-scale streaming settings, where both computational efficiency and memory footprint are critical. We also  report the scalability analysis of DSRD on synthetic graphs with varying numbers of edges (within one batch) in~\cref{tab:scalability}, detailing the computational time and memory consumption.

\begin{table}[!ht]
\centering
\caption{Scalability analysis of DSRD on synthetic graphs.}
\label{tab:scalability}
\fontsize{9pt}{12pt}\selectfont
\begin{tabular}{@{}c|ccc|c@{}}
\toprule
\textbf{Number of Edges} & \textbf{Forward (ms)} & \textbf{Backward (ms)} & \textbf{Total (ms)} & \textbf{Memory (MB)} \\
\midrule
10,000        & $57.36 \pm 5.10$     & $36.40 \pm 1.67$    & $93.75 \pm 5.23$      & 4,002.9 \\
100,000       & $634.68 \pm 104.06$  & $438.39 \pm 73.70$   & $1073.07 \pm 101.84$  & 11,062.2 \\
1,000,000     & $11397.65 \pm 346.14$& $320.59 \pm 194.18$ & $11718.24 \pm 532.95$ & 85,907.8 \\
\bottomrule
\end{tabular}
\end{table}

\subsection{Additional Results of Ablation Study}
The additional results of ablation experiments are shown in Figures~\ref{fig:ablation-AP} and~\ref{fig:ablation-AUC}.

We can observe that removing the entire DSRD block leads to substantial degradation across all datasets, particularly on high-density graphs such as Enron and Social Evo., confirming that the retentive state mechanism is essential for capturing complex interaction patterns. Disabling adaptive temporal decay causes notable drops on discrete-time graphs with coarse granularity (e.g., Flights, US Legis., UN Trade, UN Vote), where high scale values (up to 15859 links per timestep) make distinguishing interaction recency crucial; continuous-time graphs with fine-grained timestamps show smaller sensitivity. Removing topological diffusion primarily affects dense graphs like Social Evo.\ and Enron, where multi-hop context is informative, while sparse graphs exhibit smaller degradation. Ablating the retentive state yields consistent drops across all settings, with more pronounced effects under inductive evaluation on graphs with longer temporal spans (e.g., UN Vote: 72 years), highlighting the importance of persistent memory for generalization. Overall, each component addresses distinct aspects of dynamic graph learning, and their combination enables robust performance across heterogeneous temporal and structural characteristics.

\label{appendix:ablation}
\begin{figure}[!ht]
    \centering
    \includegraphics[width=0.85\linewidth]{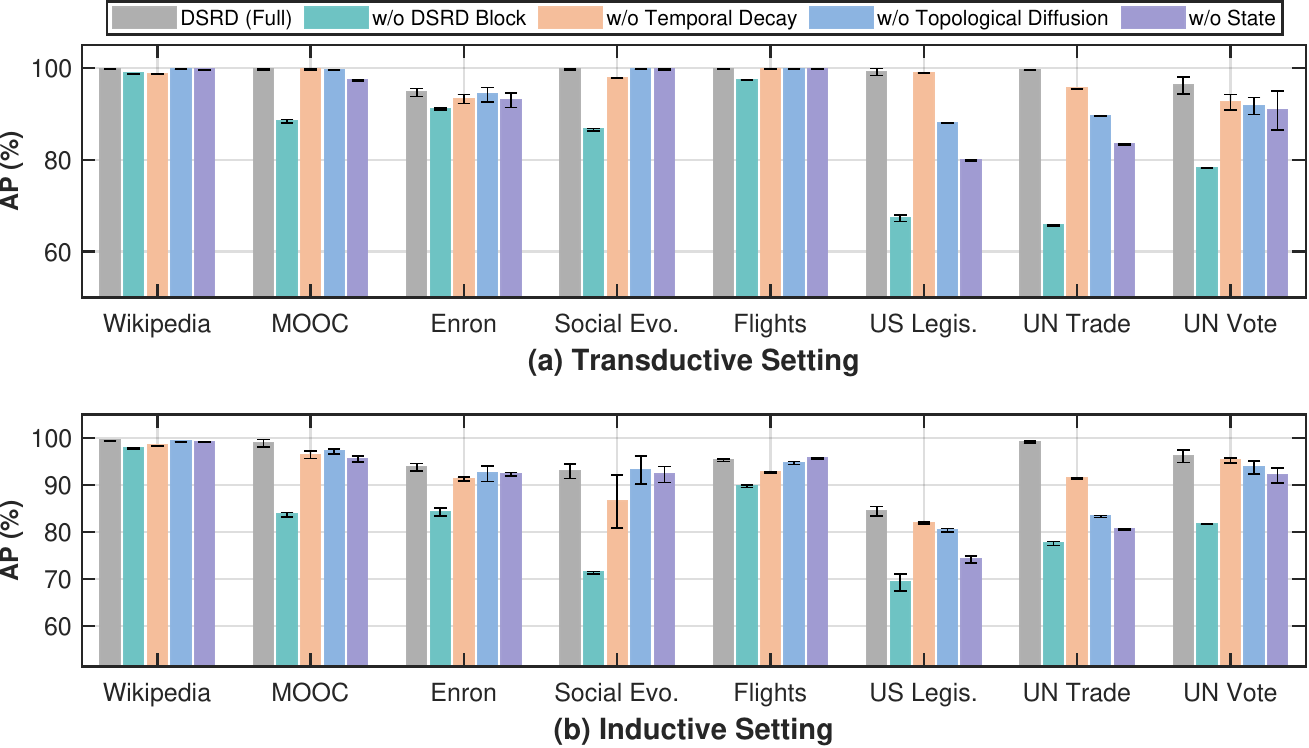}
    \caption{Ablation study on AP (\%) under (a) transductive and (b) inductive settings.}
    \label{fig:ablation-AP}
\end{figure}

\begin{figure}[!ht]
    \centering
    \includegraphics[width=0.85\linewidth]{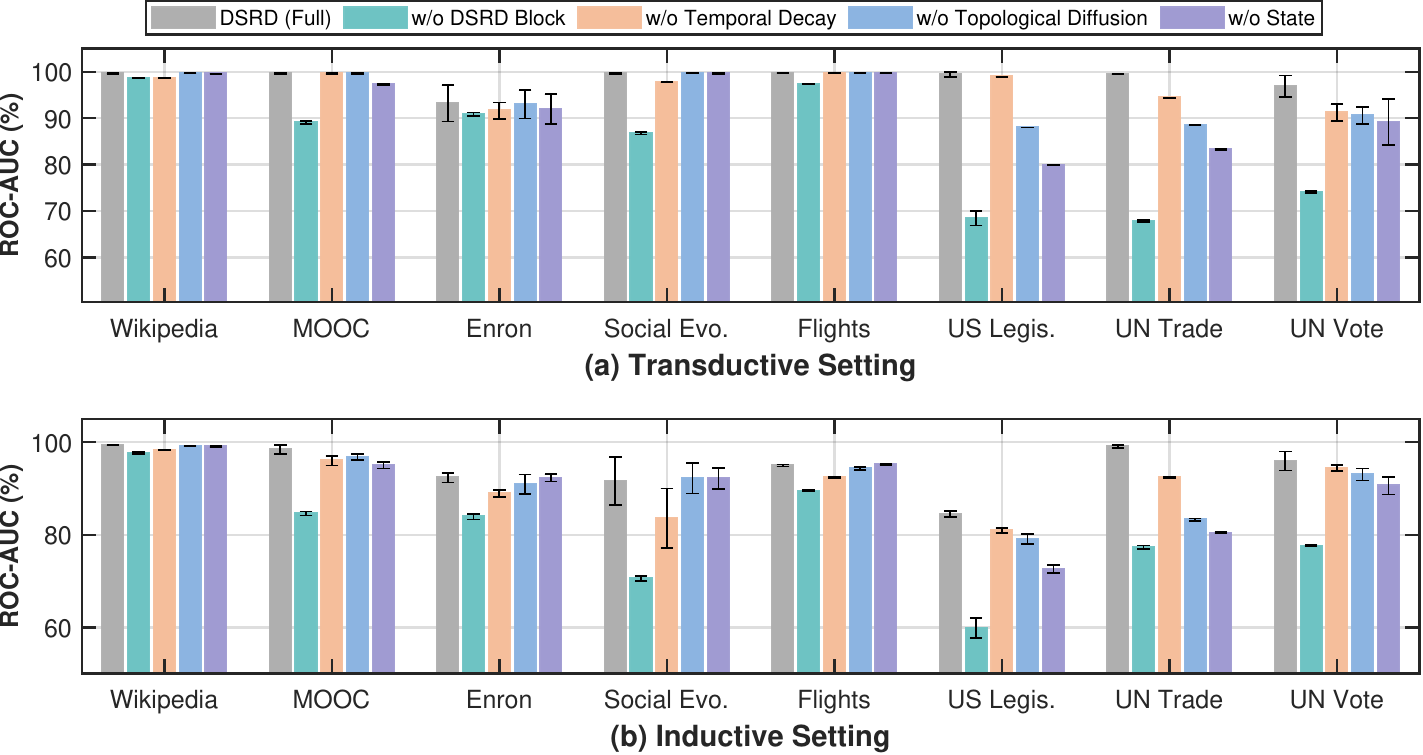}
    \caption{Ablation study on ROC-AUC (\%) under (a) transductive and (b) inductive settings.}
    \label{fig:ablation-AUC}
\end{figure}

\section{Additional Related Work}
\label{appendix:related-work}
\subsection{Structural Propagation and High-Order Dependencies}
Beyond temporal aspects, dynamic graph models must capture structural dependencies that develop over multiple hops~\cite{liu2024self,lu2024tpnet}. Traditional message-passing GNNs propagate along immediate edges, but dynamic scenarios benefit from exploring high-order connections that are not apparent from one-hop interactions~\cite{wang2022tpgnn,yu2023dygformer}. Various strategies have been proposed to address this: CAWN performs guided random walks to collect temporal neighborhoods of higher-order contacts~\cite{luo2022nat}; DyGFormer leverages self-attention over co-occurrence history to capture long-range dependencies without strict locality constraints~\cite{yu2023dygformer}; NAT encodes joint neighbor subgraphs to reflect motifs like triadic closure. Despite these advances, prior approaches often enforce rigid propagation depths or entangle structural and temporal signals in ways that obscure important patterns. Discrete-time methods that aggregate interactions within time windows risk conflating distinct causal orderings, while continuous-time frameworks tend to combine temporal and topological updates, limiting their ability to distinguish complex temporal substructures~\cite{gong2024causal,wang2021cawn,lu2024tpnet}. Recent theoretical work further shows that standard temporal GNN architectures fail to capture certain high-order properties and do not strictly improve in expressiveness with deeper stacking~\cite{walkega2025expressive}. These limitations motivate a framework that decouples temporal and structural modeling while adaptively controlling propagation depth.

\subsection{Simplification, Scalability, and Generalization Gaps}
As dynamic graphs grow in scale and complexity, there is a push towards simplifying models for efficiency and examining their generalization limits~\cite{looks2017deep,cong2023graphmixer,feng2025comprehensive}. One trend is the development of lightweight architectures that avoid heavy attention or recurrence. GraphMixer exemplifies this by using a simple MLP-Mixer style design with fixed-time positional encodings and token-mixing layers, achieving competitive temporal representations without the cost of graph attention~\cite{velivckovic2017gat} or RNN modules~\cite{cong2023graphmixer}. Such simplifications not only improve training speed and stability, but also reduce model complexity, which can help mitigate overfitting to idiosyncratic temporal patterns. For instance, scalability-focused methods like NAT introduce specialized data structures (the N-cache) and neighborhood sampling techniques to handle large dynamic graphs. By storing and updating per-node neighbor dictionaries, NAT can construct higher-order features on the fly with minimal overhead, yielding up to an order-of-magnitude speed-up over prior temporal GNNs that rely on exhaustive neighbor searches or repeated random walks~\cite{luo2022nat}. \citet{poursafaei2022towards} showed that a naively simple method, EdgeBank, which memorizes all observed edges in a lookup table, can outperform many sophisticated models on standard benchmarks. EdgeBank simply predicts a future link if it has been seen historically (and otherwise predicts no link), yet it achieved second-best rankings against state-of-the-art temporal GNNs in certain evaluation settings. This surprising result suggests that many dynamic link prediction benchmarks are dominated by repeated interactions, allowing pure memorization to excel~\cite{poursafaei2022towards}. Consequently, complex models may inadvertently be learning to memorize historical edges rather than truly extrapolate temporal-structural patterns~\cite{martinez2016survey}. When evaluation conditions are made more challenging, for example, predicting interactions that have not occurred before or using harder negative samples, the performance of memorization-based methods drops sharply, and models overly reliant on such memory also suffer disproportionately. These observations highlight a generalization gap: methods that perform well by exploiting easy patterns (recurrent edges) may falter on scenarios requiring inductive reasoning or adaptation to new dynamics. Therefore, an emerging theme in recent work is not only to improve scalability and efficiency, but also to design evaluation protocols and models that reward genuine temporal reasoning over simple historical lookup. Bridging this gap remains an open challenge, calling for dynamic graph learners that are both robust across diverse temporal regimes and capable of generalizing beyond the training history.

\section{Limitations and Future Prospects}
\label{appendix:concerns}

\paragraph{Limitations.}
While DSRD demonstrates strong empirical performance across diverse benchmarks, several limitations warrant discussion. First, the current architecture assumes homogeneous node dynamics, applying the same learnable decay kernels across all entities. In heterogeneous dynamic graphs where different node types (e.g., users vs. items, or different organizational roles) exhibit fundamentally different temporal behaviors, a type-aware parameterization may yield further improvements. Second, although the gating mechanism adaptively balances short-term and long-term signals, it does not explicitly model periodic or seasonal patterns. For domains with strong cyclical dynamics (e.g., daily commuting patterns or annual trade cycles), incorporating Fourier-based temporal encodings or explicit periodicity modeling could enhance representational capacity. Third, the second-order state representation, while expressive, incurs $\mathcal{O}(d^2)$ memory per node per layer, which may become prohibitive for extremely large-scale graphs with high embedding dimensions.

\paragraph{Future Prospects.}
Dynamic graph learning methods, including DSRD, can be applied to domains with significant societal implications. On the positive side, improved temporal link prediction can enhance recommendation systems, fraud detection, and epidemic forecasting. However, such capabilities also raise concerns about privacy (e.g., inferring undisclosed relationships), surveillance, and potential misuse in tracking individual behaviors over time. We encourage practitioners to consider ethical guidelines and privacy-preserving techniques when deploying dynamic graph models in sensitive applications. Several promising directions emerge from this work. First, extending the dual-scale framework to heterogeneous dynamic graphs with type-specific decay mechanisms could broaden applicability. Second, integrating explicit periodicity modeling or leveraging neural ordinary differential equations for continuous-time dynamics may capture more complex temporal patterns. Third, exploring connections between retentive dynamics and state-space models (e.g., Mamba~\cite{gu2024mamba}) could yield more efficient architectures for streaming graph data. Finally, developing theoretical frameworks that characterize the expressiveness of adaptive decay mechanisms relative to fixed-decay baselines remains an important open question.

%%%%%%%%%%%%%%%%%%%%%%%%%%%%%%%%%%%%%%%%%%%%%%%%%%%%%%%%%%%%%%%%%%%%%%%%%%%%%%%
%%%%%%%%%%%%%%%%%%%%%%%%%%%%%%%%%%%%%%%%%%%%%%%%%%%%%%%%%%%%%%%%%%%%%%%%%%%%%%%

\end{document}